\RequirePackage[l2tabu, orthodox]{nag}
\documentclass[version=3.21, pagesize, twoside=off, bibliography=totoc, DIV=calc, fontsize=12pt, a4paper]{scrartcl}
\input{glyphtounicode}
\usepackage[T1]{fontenc}
\usepackage[utf8]{inputenc}
\usepackage{newunicodechar}
\usepackage{lmodern}
\usepackage{textcomp}
\DeclareFontShape{OMX}{cmex}{m}{n}{
  <-7.5> cmex7
  <7.5-8.5> cmex8
  <8.5-9.5> cmex9
  <9.5-> cmex10
}{}

\DeclareFontFamily{U} {MnSymbolA}{}
\DeclareFontShape{U}{MnSymbolA}{m}{n}{
  <-6> MnSymbolA5
  <6-7> MnSymbolA6
  <7-8> MnSymbolA7
  <8-9> MnSymbolA8
  <9-10> MnSymbolA9
  <10-12> MnSymbolA10
  <12-> MnSymbolA12}{}
\DeclareFontShape{U}{MnSymbolA}{b}{n}{
  <-6> MnSymbolA-Bold5
  <6-7> MnSymbolA-Bold6
  <7-8> MnSymbolA-Bold7
  <8-9> MnSymbolA-Bold8
  <9-10> MnSymbolA-Bold9
  <10-12> MnSymbolA-Bold10
  <12-> MnSymbolA-Bold12}{}

\DeclareSymbolFont{MnSyA} {U} {MnSymbolA}{m}{n}

\DeclareMathSymbol{\rightlsquigarrow}{\mathrel}{MnSyA}{160}

\SetSymbolFont{largesymbols}{normal}{OMX}{cmex}{m}{n}
\SetSymbolFont{largesymbols}{bold}  {OMX}{cmex}{m}{n}

\usepackage{booktabs}
\usepackage{calc}
\usepackage{tabularx}

\usepackage{xpatch} 
\newtoggle{LCpres}
\togglefalse{LCpres}

\usepackage{mathtools} 
	\mathtoolsset{showonlyrefs,showmanualtags}

\usepackage[round]{natbib}
\usepackage[french, english]{babel}
\usepackage{listings} 
\usepackage[nolist,smaller,printonlyused]{acronym}
\usepackage{fmtcount}
\usepackage[nodayofweek]{datetime}
\nottoggle{LCpres}{
	\usepackage[textsize=small]{todonotes}
}{
}
\makeatletter
\iftoggle{LCpres}{
	\usepackage{hyperref}
}{
	\usepackage[pdfusetitle]{hyperref}
	\pdfstringdefDisableCommands{%
		\let\thanks\@gobble
	}
}
\makeatother
\nottoggle{LCpres}{
	\usepackage{authblk}
	
}{
}
\usepackage{silence}
\hypersetup{breaklinks, bookmarksopen, colorlinks, linkcolor=black, citecolor=black, urlcolor={blue!80!black}}
\iftoggle{LCpres}{
	\hypersetup{hyperfigures}
}{
}

\usepackage{bookmark}
	\ExplSyntaxOn
	\seq_new:N \g_oc_hrauthor_seq
	\NewDocumentCommand{\addhrauthor}{m}{
		\seq_gput_right:Nn \g_oc_hrauthor_seq { #1 }
	}
	\DeclareExpandableDocumentCommand{\hrauthor}{}{
		\seq_use:Nn \g_oc_hrauthor_seq {,~}
	}
	\ExplSyntaxOff
	{
		\catcode`#=11\relax
		\gdef\fixauthor{\xpretocmd{\author}{\addhrauthor{#2}}{}{}}%
	}
	\nottoggle{LCpres}{
		\usepackage{authblk}
		
		\fixauthor
		\AtBeginDocument{
		    \hypersetup{pdfauthor={\hrauthor}}
		}
	}{
	}

\nottoggle{LCpres}{
	\usepackage{enumitem} 
}{
}
\usepackage{ragged2e} 
\usepackage{siunitx} 
\usepackage{braket} 
\usepackage{doi}

\usepackage{amsmath,amsthm}
\usepackage{amssymb}
\usepackage{bm}

\usepackage{environ}
\usepackage[french,english]{cleveref}
\usepackage{pgf}
\usepackage{pgfplots}
	\usetikzlibrary{babel, matrix, fit, plotmarks, calc, trees, shapes.geometric, positioning, plothandlers, arrows, shapes.multipart}
\pgfplotsset{compat=1.14}
\usepackage{graphicx}

\DeclareMathAlphabet\mathbfcal{OMS}{cmsy}{b}{n}

\graphicspath{{graphics/},{graphics-dm/}}
\DeclareGraphicsExtensions{.pdf}
\usepackage{printlen}
\uselengthunit{mm}
\usepackage{scrhack}
\usepackage{mathrsfs}
\DeclareFontFamily{U}{rsfs}{\skewchar\font127 }
\DeclareFontShape{U}{rsfs}{m}{n}{%
   <-6.5> rsfs5
   <6.5-8> rsfs7
   <8-> rsfs10
}{}

\iftoggle{LCpres}{
	\usepackage{appendixnumberbeamer}
	\setbeamertemplate{navigation symbols}{} 
	\usepackage{preamble/beamerthemeParisFrance}
	\usefonttheme{professionalfonts}
	\setcounter{tocdepth}{10}
	\input{preamble/beamer-patch.tex}
}{
}

\newcommand{\R}{ℝ}
\newcommand{\N}{ℕ}

\newcommand{\suchthat}{\;\ifnum\currentgrouptype=16 \middle\fi|\;}

\AtBeginDocument{%
	\renewcommand{\epsilon}{\varepsilon}
}

\newcommand{\restr}[2]{{#1|}_{#2}}

\newcommand{\allalts}{\mathcal{A}}

\newcommand{\alts}{A}
\newcommand{\alt}{a}

\newcommand{\prof}{\mathbf{R}}

\DeclareDocumentCommand{\lato}{ O{\prof} O{\alts} }{[#1 \mapsto #2]}
\newcommand{\tightoverset}[2]{%
  \mathop{#2}\limits^{\vbox to -.5ex{\kern-0.9ex\hbox{$#1$}\vss}}}
\DeclareDocumentCommand{\latoin}{ O{\prof} O{\alpha} }{[#1 \tightoverset{\in}{⟼} #2]}

\newcommand{\SGE}{S^{\mathcal{GE}}}

\newcommand{\AF}{\mathcal{AF}}
\newcommand{\labelling}{\mathcal{L}}
\newcommand{\labin}{\textbf{in}}
\newcommand{\labout}{\textbf{out}}
\newcommand{\labund}{\textbf{undec}}
\newcommand{\nonemptyor}[2]{\ifthenelse{\equal{#1}{}}{#2}{#1}}
\newcommand{\gextlab}[2][]{
	\labelling{\mathcal{GE}}_{(#2, \nonemptyor{#1}{\ibeatsr{#2}})}
}
\newcommand{\allargs}{S^*}
\newcommand{\args}{S}
\newcommand{\clargs}{S_γ}
\newcommand{\ar}{s}
\newcommand{\clar}{s_γ}
\newcommand{\ext}{\mathcal{E}}

\newcommand{\ileadsto}{\mathbin{\rightlsquigarrow}}

\newcommand{\ileadstoinv}{\mathbin{\rightlsquigarrow^{-1}}}
\newcommand{\ileadstor}[1]{\mathbin{\restr{\rightlsquigarrow}{#1}}}
\newcommand{\mleadsto}[1][\eta]{\mathbin{\rightlsquigarrow_{#1}}}
\newcommand{\mleadstoinv}{\mathbin{\rightlsquigarrow_\eta^{-1}}}
\newcommand{\ibeats}{\mathbin{\vartriangleright}}
\newcommand{\ibeatse}{\mathbin{\vartriangleright_\exists}}
\newcommand{\ibeatsst}{\mathbin{\vartriangleright_\forall}}

\newcommand{\nibeats}{\mathbin{⋫}}
\newcommand{\nibeatse}{\mathbin{⋫_\exists}}

\newcommand{\nibeatsst}{\mathbin{⋫_\forall}}

\newcommand{\ibeatsr}[1]{\mathbin{\restr{{\vartriangleright}}{#1}}}
\newcommand{\nibeatsr}[1]{\mathbin{\restr{{⋫}}{#1}}}
\newcommand{\ibeatsinv}{\mathbin{\vartriangleright^{-1}}}
\newcommand{\ibeatseinv}{\mathbin{\vartriangleright_\exists^{-1}}}
\newcommand{\mbeats}[1][\eta]{\mathbin{\vartriangleright_{#1}}}
\newcommand{\mbeatsinv}{\mathbin{\vartriangleright_\eta^{-1}}}

\nottoggle{LCpres}{
	\newtheorem{definition}{Definition}
	\newtheorem{theorem}{Theorem}
	\newtheorem{lemma}{Lemma}
	\newtheorem{condition}{Condition}

\theoremstyle{remark}
	\newtheorem{examplex}{Example}
	\newtheorem{remarkx}{Remark}

\newenvironment{example}{
	\pushQED{\qed}\examplex
}{
	\popQED\endexamplex
}
\newenvironment{remark}{
	\pushQED{\qed}\remarkx
}{
	\popQED\endremarkx
}
}{
}
\crefname{examplex}{example}{examples}
\crefname{condition}{condition}{conditions}
\makeatletter
\cref@addlanguagedefs{french}{%
	\crefname{examplex}{exemple}{exemples}%
}
\makeatother

\apptocmd{\sloppy}{\hbadness 10000\relax}{}{}

\makeatletter
\patchcmd{\@doi}{dx.doi.org}{doi.org}{}{}
\makeatother

\makeatletter
\newcommand{\boldor}[2]{%
	\ifnum\strcmp{\f@series}{bx}=\z@
		#1%
	\else
		#2%
	\fi
}

\makeatother

\newunicodechar{ℕ}{\mathbb{N}}
\newunicodechar{ℝ}{\mathbb{R}}
\newunicodechar{−}{\ifmmode{-}\else\textminus\fi}
\newunicodechar{≠}{\ensuremath{\neq}}
\newunicodechar{≤}{\ensuremath{\leq}}
\newunicodechar{≥}{\ensuremath{\geq}}
\newunicodechar{≻}{\succ}
\newunicodechar{⊁}{\nsucc}
\newunicodechar{▷}{\triangleright}
\newunicodechar{⋫}{\ntriangleright}
\newunicodechar{→}{\ifmmode\rightarrow\else\textrightarrow\fi}
\newunicodechar{⇒}{\ensuremath{\Rightarrow}}
\newunicodechar{⇏}{\ensuremath{\nRightarrow}}
\newunicodechar{⇔}{\ensuremath{\Leftrightarrow}}
\newunicodechar{∪}{\cup}
\newunicodechar{∩}{\cap}
\newunicodechar{∧}{\land}
\newunicodechar{∨}{\lor}
\newunicodechar{¬}{\ifmmode\lnot\else\textlnot\fi}
\newunicodechar{…}{\ifmmode\ldots\else\textellipsis\fi}
\newunicodechar{×}{\ifmmode\times\else\texttimes\fi}
\newunicodechar{γ}{\ensuremath{\gamma}}
\newunicodechar{□}{\Box}

\newlength{\GraphsNodeSep}
\newlength{\MCDSCatHeight}
\newlength{\MCDSAltHeight}
\newlength{\MCDSAltSep}
\newlength{\MCDSCatWidth}
\newlength{\MCDSEvalRowHeight}
\newlength{\MCDSAltsToCatsSep}
\newlength{\MCDSArrowDownOffset}
\tikzset{/Graphs/dot/.style={
	shape=circle, fill=black, inner sep=0, minimum size=1mm
}}
\tikzset{/MC/D/S/alt/.style={
	shape=rectangle, draw=black, inner sep=0, minimum height=\MCDSAltHeight, minimum width=2.5cm, anchor=north east
}}
\tikzset{MC/D/S/pref/.style={
	shape=ellipse, draw=gray, thick
}}
\tikzset{/MC/D/S/cat/.style={
	shape=rectangle, draw=black, inner sep=0, minimum height=\MCDSCatHeight, minimum width=\MCDSCatWidth, anchor=north west
}}
\tikzset{/MC/D/S/evals matrix/.style={
	matrix, row sep=-\pgflinewidth, column sep=-\pgflinewidth, nodes={shape=rectangle, draw=black, inner sep=0mm, text depth=0.5ex, text height=1em, minimum height=\MCDSEvalRowHeight, minimum width=12mm}, nodes in empty cells, matrix of nodes, inner sep=0mm, outer sep=0mm, row 1/.style={nodes={draw=none, minimum height=0em, text height=, inner ysep=1mm}}
}}

\tikzset{/GUI/button/.style={
	rectangle, very thick, rounded corners, draw=black, fill=black!40
}}

\newlength{\BDNodeSep}
\newlength{\BDDecLength}
\newlength{\BDDecWidth}
\tikzset{/Beliefs/attacker/.style={
	shape=rectangle, draw, minimum size=8mm
}}
\tikzset{/Beliefs/supporter/.style={
	shape=circle, draw
}}
\tikzset{/Beliefs/undefeated/.style={
	shape=circle, draw
}}
\tikzset{/Beliefs/defeated/.style={
	shape=rectangle
}}
\tikzset{/Beliefs/attack/.style={
	arrows=-{>}
}}
\tikzset{/Beliefs/attackst/.style={
	arrows=-{>}, double
}}
\tikzset{/Beliefs/attackun/.style={
	arrows=-{>},
	decorate,
	decoration={coil, segment length=1.5mm, aspect=0},
}}
\tikzset{/Beliefs/nattack/.style={
	arrows=-{>},
	decoration={
		markings, mark=at position 0.5 with {
			\draw[-] ++ (-1mm, 1mm) -- (1mm, -1mm);
		}
	},
	postaction={decorate},
}}
\tikzset{/Beliefs/decisive/.style={
	alias=thisone,
	append after command={
      (thisone.south) edge[draw] +(0, -\BDDecLength) ++(0, -\BDDecLength) edge[draw] ++(-\BDDecWidth, 0) edge[draw] ++(\BDDecWidth, 0)
	}
}}

\begin{acronym}
\acro{DP}{Deliberated Preference}
\acro{DJ}{Deliberated Judgment}
\acro{AS}{argumentative stance}
\acro{AMCD}{Aide Multicritère à la Décision}
\acro{ASA}{Argument Strength Assessment}
\acro{DA}{Decision Analysis}
\acro{DM}{Decision Maker}
\acro{DPr}{Deliberated Preferences}
\acro{DRSA}{Dominance-based Rough Set Approach}
\acro{DSS}{Decision Support Systems}
\acrodefplural{DSS}{Decision Support Systems}
\acro{EJOR}{European Journal of Operational Research}
\acro{LNCS}{Lecture Notes in Computer Science}
\acro{MCDA}{Multicriteria Decision Aid}
\acro{MIP}{Mixed Integer Program}
\acro{NCSM}{Non Compensatory Sorting Model}
\acro{PL}{Programme Linéaire}
\acro{PLNE}{Programme Linéaire en Nombres Entiers}
\acro{PM}{Programme Mathématique}
\acro{MP}{Mathematical Program}
\acro{MIP}{Mixed Integer Program}
\acrodefplural{PM}{Programmes Mathématiques}
\acro{PMML}{Predictive Model Markup Language}
\acro{RESS}{Reliability Engineering \& System Safety}
\acro{SMAA}{Stochastic Multicriteria Acceptability Analysis}
\acro{URPDM}{Uncertainty and Robustness in Planning and Decision Making}
\acro{XML}{Extensible Markup Language}
\end{acronym}

\newcommand{\argscldec}{S_{γ\textnormal{dec}}}
\newcommand{\argscldef}{S_{γ\textnormal{def}}}
\newcommand{\argsclres}{S_{γ\textnormal{res}}}
\newcommand{\argsreplclres}{E_{γ\textnormal{res}}}
\newcommand{\argsreplcldec}{E_{γ\textnormal{dec}}}
\newcommand{\argsrreplcldec}{R_{γ\textnormal{dec}}}

\title{A formal framework for deliberated judgment%
	\thanks{This is the postprint version of the article published in Theory and Decision, \url{https://doi.org/10.1007/s11238-019-09722-7}. 	The text is identical, except for minor wording modifications.}
}
\author{Olivier Cailloux}
\author{Yves Meinard}
\affil{Université Paris-Dauphine, PSL Research University, CNRS, LAMSADE, 75016 PARIS, FRANCE\\
	\href{mailto:olivier.cailloux@dauphine.fr}{olivier.cailloux@dauphine.fr}
}
\hypersetup{
	pdfsubject={preference},
	pdfkeywords={decision analysis, justification, empirical validation, methodology}
}

\begin{document}
\maketitle

\begin{abstract}
While the philosophical literature has extensively studied how decisions relate to arguments, reasons and justifications, decision theory almost entirely ignores the latter notions.
In this article, we elaborate a formal framework in order to introduce in decision theory the stance that decision-makers take towards arguments and counter-arguments.We start from a decision situation, where an individual requests decision support. We formally define, as a commendable basis for decision-aid, this individual’s deliberated judgment, a notion inspired by Rawls' contributions to the philosophical literature, and embodying the requirement that the decision-maker should carefully examine arguments and counter-arguments.  
We explain how models of deliberated judgment can be validated empirically.
We then identify conditions upon which the existence of a valid model can be taken for granted, and analyze how these conditions can be relaxed.
We then explore the significance of our framework for the practice of decision analysis.
Our framework opens avenues for future research involving both philosophy and decision theory, as well as empirical implementations. 
\end{abstract}

\section{Introduction}
\label{sec:intro}

Introducing their “reason-based theory of choice”, \citet{dietrich_reason-based_2013} noticed that, although the philosophical literature has largely illustrated the usefulness of the concepts of reasons and arguments to think through action and decisions, decision theory strives to account for the latter exclusively in terms of preferences and beliefs. Despite \citeauthor{dietrich_reason-based_2013}’s \citeyearpar{dietrich_reason-based_2013, dietrich_reason-based_2016} efforts, the gap remains large between philosophical and choice theoretic approaches.

This gap echoes a classical dichotomy in “moral sciences” between, on the one hand, first-person justifications of one's acts in terms of reasons and arguments structuring these reasons, and on the other hand, third-person representations in terms of beliefs and preferences \citep{hausman_preference_2011}. By neglecting reason-based and other argumentative accounts, decision theory tends to devalue decision-makers' understanding of their own actions.

This gap has tended to insulate decision theory from important philosophical debates in the past thirty to forty years. Among the most influential approaches in these debates, \citet{scanlon_what_2000} highlighted the links between reasons, justification and moral notions such as fairness and responsibility, \citeauthor{habermas_theorie_1981}’ \citeyearpar{habermas_theorie_1981} “theory of communicative action” articulated the importance of justification and argumentation as distinctive features of rational action, and \citet{rawls_political_2005} launched the debates on the “acceptability” \citep{estlund_democratic_2009} of reasons and arguments for public justification.

This gap also has important practical implications for decision analysis, by complicating the task for analysts to explain the recommendations they give to their clients. This, in turn, casts doubts on these recommendations, which appear to be imposed to rather than endorsed by decision-makers.

In this article, we aim to participate in unlocking this situation, by elaborating a framework designed to allow decision analysts to provide recommendations that decision-makers truly endorse, in empirical reality.

For that purpose, we introduce, as a commendable basis for recommendation, the “deliberated judgments” of the decision-maker. Roughly stated, these “deliberated judgments” represent the propositions that the decision-maker will consider to be well-grounded, if he duly takes into account all the relevant arguments. This concept is inspired by \citeauthor{goodman_fact_1983}’s \citeyearpar{goodman_fact_1983} and \citeauthor{rawls_theory_1999}’ \citeyearpar{rawls_theory_1999} notion of reflective equilibrium. It also owes much to \citeauthor{roy_multicriteria_1996}'s \citeyearpar{roy_multicriteria_1996} view that an important part of the decision support interaction consists, for the analyst, in ensuring that the aided individual understands and accepts the reasoning on which the prescription is based. 

This article is organized as follows. In \cref{sec:core}, we define our core concepts, including the central concept of deliberated judgments. In \cref{sec:empirical}, we then explore the issue of how empirical data come into play and are involved in the validation of models. This illustrates the empirical aspect of our framework, which distinguishes it from standard prescriptive approaches. Obviously enough, at this stage, the pivotal issue is to determine how one can say anything about “deliberated judgments”, given that, for any non-trivial decision, the potentially relevant arguments are infinitely numerous. Lastly, \cref{sec:discussion} discusses the significance of our approach for the practice of decision analysis and outlines future empirical applications. 

\section{Core concepts and notations}
\label{sec:core}

In this section, we start by presenting the general setting of our approach, including our understanding of arguments and of the topic on which the individual aims to take a stance. We then introduce our formalization of argumentative disposition, capturing an individual's attitude towards arguments. This eventually allows us to present our notion of “deliberated judgment”.

\subsection{General setting}
Our approach starts from and is largely structured by the point of view of decision-analysis. We accordingly assume that a decision situation has been identified: we admit that there is an individual $i$ who requests decision support to answer questions such as: “is action $a$ better than action $b$?”, or “which beliefs should I have about such or such matter?”. We consider that a topic $T$ -- a set of propositions on which the decision analysis process aims to lead the decision-maker to take a stance -- is defined.%
\footnote{We remain at a fairly abstract level in our conceptualization of the topic. We accordingly set aside all the issues concerning the construction of problems and the evolution of their meaning as the decision process unfolds in concrete decision situations \citep{rosenhead_rational_2001}.}
We do not formally define propositions and simply understand the notion in its ordinary sense. For example, a proposition can be a claim spelled out in a text in a natural language, such as the claim that action $a$ is the most appropriate action for $i$ in a given decision situation.

We also consider arguments that can be used by $i$ to make up her mind about propositions in $T$. 
Here we understand the notion of argument in a large sense: anything that can be used to support a proposition, or undermine the effectiveness of such a support, is an argument. In the latter case, we talk about a counter-argument. Arguments as we understand them can encompass a huge diversity, ranging from very basic arguments that can be stated in a couple of words, to intricate arguments embedding numerous sub-arguments associated to one another in complex ways.

Let us then define the set $\allargs$ that contains all the arguments that one uses when trying to make up one’s mind about $T$. 
$\allargs$ can be understood in a “pragmatic” sense, as the set of all the arguments available around the temporal window of the decision process. It can also be understood in an “idealistic” sense, as the set of all the arguments that can possibly be raised, including those that humankind has not yet discovered.%
\footnote{Because no one has a concrete access to such an idealistic set of all the arguments, we expect that this concept will be mainly useful for philosophical explorations, and that the pragmatic interpretation will prevail in practical applications.}

Observe that under both interpretations, in all decision situations but the most trivial, it will be untenable to assume that the analyst knows all of $\allargs$: the analyst will only know a strict subset $\args \subset \allargs$, containing the arguments that she has been able to gather.%
\footnote{Even under the pragmatic interpretation, claiming that $\args = \allargs$ would mean that there is no relevant knowledge beyond what the analyst can find by studying the literature and consulting experts and stakeholders, but also that the list of arguments she has found captures all the semantic and linguistic subtleties that could distinguish alternative formulations of arguments.}
An important part of our work in this article will be to identify conditions allowing to draw conclusions relating to $\allargs$ despite the fact that no one ever knows more than a strict subset of $\allargs$. 

\begin{example}[Ranking]
Let us simply illustrate the content of the concepts introduced so far. Let $\allalts$ be a set of alternatives that $i$ is interested in ranking. For all $\alt_1 ≠ \alt_2 \in \allalts$, define $t_{\alt_1 \succ \alt_2}$ as the sentence: “$\alt_1$ ought to be ranked above $\alt_2$”, and $t_{\alt_1 \sim \alt_2}$ as “$\alt_1$ ought to be ranked \emph{ex-æquo} with $\alt_2$”. Define $T = \bigcup_{\alt_1 ≠ \alt_2 \in \allalts} \{t_{\alt_1 > \alt_2}, t_{\alt_1 \sim \alt_2}\}$ as the set of all such sentences. The topic $T$ represents the propositions on which $i$ is interested to make up her mind. Define $\allargs$ as the set of all strings corresponding to sentences in English. This set contains formulations of all the arguments that people can think about and use to make up their mind about the topic, and much more. An example of an argument is $\ar = $ “Alternative $\alt_1$ ought to be ranked above $\alt_2$ because $\alt_1$ is better than $\alt_2$ on every criterion relevant to this problem”.
\end{example}

Our aim in the remainder of this section is to define formally $i$’s perspective towards the topic after he has considered all the arguments that are possibly relevant to the situation. We term this: $i$'s \ac{DJ}.

\subsection{Argumentative disposition}
\label{sec:AS}
To define $i$’s \ac{DJ}, we need to capture $i$'s attitude towards arguments. Importantly, we also need to capture the fact that $i$ may change her opinion about arguments and their relative strengths. She can change her mind because of reasons independent of her endeavor to tackle the problem she addresses, for example depending on her mood.
More interestingly, $i$ will possibly change her mind when confronted with new arguments. For example, imagine that $i$ has heard about two arguments, $\ar_1$ and $\ar_2$, and she thinks that $\ar_2$ turns $\ar_1$ into an ineffective argument. But then she comes to realize that $\ar_2$ is in turn rendered ineffective by a third argument, $\ar_3$. After having thought about $\ar_3$, it might be that $i$ no longer considers that $\ar_2$ undermines $\ar_1$.

Note that for simplicity's sake, we say that an argument becomes ineffective (because of another argument) to mean that it becomes ineffective in its ability to support some proposition or to render other arguments ineffective.

Le us introduce our formalism to account for such a situation.%
\footnote{Our approach to formalize this concept is inspired by formal argumentation theory in artificial intelligence \citep{dung_acceptability_1995, rahwan_argumentation_2009}. However, the latter approach is not sufficient to empirically investigate $i$'s attitude towards arguments, because it neglects two crucial tasks.
First, this literature does not investigate the role that the decision analyst plays when she interacts with a decision-maker: should she remain a neutral observer, or should she interact more tightly with the decision-maker by providing him with arguments and counter-arguments liable to lead him to change his mind? 
Second, this literature does not put emphasis on the specific challenges involved in interacting with a decision-maker to identify empirically the arguments he endorses. Most of the time, this literature considers situations where the relation between arguments can be computed from a given logical representation of the arguments \citep{besnard_elements_2008} or is given \emph{a priori} \citep{baroni_semantics_2009}, possibly integrating uncertainties \citep{hunter_probabilistic_2014} and dynamics \citep{rotstein_dynamic_2010, marcos_dynamic_2011, dimopoulos_control_2018}.

Its most common use assumes that it is possible to establish the objective relations between arguments. In our example, $\ar_3$ would be considered to objectively attack $\ar_2$ and $\ar_2$ to objectively attack $\ar_1$. However, in some cases, it might be difficult, or perhaps even impossible, to determine such objective relations. In any case, this distinction is superfluous if the goal is to inquire about $i$’s opinion about these relations between arguments.
Other proposals in formal argumentation theory \citep{amgoud_reasoning_2002, bench-capon_persuasion_2003, amgoud_making_2008, amgoud_using_2009, bench-capon_abstract_2009, ferretti_approach_2017} supplement an objective attack relation with information representing $i$’s subjectivity, such as his values or his preference over arguments. Such approaches seem closer to our aim, but they also use an objective attack relation, in addition to the subjective information. Furthermore, this approach assumes that it is possible to distinguish between, on the one hand, cases where $\ar_3$ attacks $\ar_2$ but $i$ does not deem this attack important, and on the other hand situations where $\ar_3$ does not attack $\ar_2$. This assumption is also unnecessary for our purpose. Because our aim is mainly empirical, we propose to use another formalism, more adapted to our specific purpose, and leave aside here the task of more fully exploring the relations with proposals in formal argumentation theory such as dynamic argumentation.}

Let us start by defining a set of possible perspectives $P$ that $i$ can have towards the topic $T$. A perspective $p \in P$ captures all the elements determining how $i$ would react to arguments in $\allargs$. In $p$, $i$ has a specific set of arguments in mind, which can partly determine his reaction to other arguments in $\allargs$. But other elements can come into play, such as (to come back to our example above) his mood. 

If the decision analyst provides $i$ with a new argument $\ar$, this might lead $i$ to switch from $p$ to another perspective $p'$ integrating both $\ar$ and the arguments that $i$ had in mind in $p$, and possibly other arguments that $i$ might have been led to construct when trying to make up his mind about $\ar$ and its implications. $i$'s perspective can also change over time, because he forgets some arguments.

We forcefully emphasize that we do not claim to be able to provide a complete account of all the elements encapsulated in this notion of perspective. In fact, our approach does not even require to believe that it is possible for anyone to capture the content of perspectives, or more generally to directly measure details about $i$'s internal states of mind. The notion of perspective merely serves as an abstract device allowing to ground the idea that $i$ may have changing attitudes towards some pairs of arguments.

Based on these notions, given $T$ and $\allargs$, define $i$'s argumentative disposition towards $T$ as $(\ileadsto, \ibeatse, \nibeatse)$. These three relations, described here below, constitute the formal primitives of our concept of argumentative disposition.
\begin{description}
	\item[$\ileadsto$] is a relation from $\allargs$ to $T$. An argument $\ar$ \emph{supports} a proposition $t$, denoted by $\ar \ileadsto t$, iff $i$ considers that $\ar$ is an argument in favor of $t$. We emphasize that this definition should be understood in a conditional sense: $\ar \ileadsto t$ means that $i$ considers that, if $\ar$ holds in her eyes, then she should endorse $t$, but this does not say anything about whether she thinks that $\ar$ holds. An argument $\ar$ may support several propositions in $i$'s view, or none.
	\item[$\ibeatse$] is a binary relation over $\allargs$ representing whether $i$ considers that a given argument trumps another one in \emph{some} perspective. Let $\ar_1, \ar_2 \in \allargs$ be two arguments. We note $\ar_2 \ibeatse \ar_1$ ($\ar_2$ \emph{trumps} $\ar_1$) iff there is at least one perspective within which $i$ considers that $\ar_2$ turns $\ar_1$ into an ineffective argument.%
	\footnote{Note that, contrary to the usual assumption in formal argumentation theory, we do not consider it possible that both $\ar_2$ trumps $\ar_1$ and $\ar_1$ trumps $\ar_2$ \emph{in a given perspective}. This is a choice of modelization, and not an hypothesis about the way $i$ thinks: for $\ar_2 \ibeatse \ar_1$ to hold, by definition of our “trump” relation, $\ar_2$ must be a sufficiently strong argument to turn $\ar_1$ into an ineffective argument. If, on the contrary, $i$ considers that $\ar_2$ is a plausible argument defending some claim incompatible with $\ar_1$, but not sufficiently strong to defeat $\ar_1$, then we model it by $\ar_2 \nibeatse \ar_1$ and $\ar_1 \nibeatse \ar_2$. Our choice permits to reduce our informational requirements, as there are fewer cases to be distinguished (our framework treats in the same way situations where two arguments trump each other and situations where none trumps the other). Note however that we do allow for the possibility that $\ar_2 \ibeatse \ar_1$ and $\ar_1 \ibeatse \ar_2$: this can happen by $i$ adopting each of those two attitudes in two different perspectives. Hence, our choice of modelization does not translate in any formal restriction. This note only serves to make the semantics of the notion encapsulated by our “trump” relation clear.}
	 Let us emphasize that we are concerned with how $i$ sees $\ar_2$ and $\ar_1$, not about whether $\ar_2$ should be considered to be a good argument to trump $\ar_1$ by any independent standard. 
	\item[$\nibeatse$] is a binary relation over $\allargs$ defined in a similar way: $\ar_2 \nibeatse \ar_1$ iff there is at least one perspective within which $i$ does not consider that $\ar_2$ turns $\ar_1$ into an ineffective argument.
\end{description}
We assume that $\forall \ar_2, \ar_1 \in \allargs: ¬(\ar_2 \ibeatse \ar_1) ⇒ \ar_2 \nibeatse \ar_1$.

We consider that it is possible to query $i$ about the trump relation between two arguments, and thus obtain information about $\ibeatse$, to the following limited extent: $i$ may be presented with two arguments, $\ar_1$ and $\ar_2$, and asked whether he thinks that $\ar_2$ trumps $\ar_1$, or $\ar_1$ trumps $\ar_2$, or neither. In any case, we consider that $i$ answers from the perspective he is currently in (to which we have no other access than through this query). Thus, if $i$ answers that $\ar_2$ trumps $\ar_1$, we know that $\ar_2 \ibeatse \ar_1$. Indeed, in such a case we know that there is at least one perspective within which he thinks that $\ar_2$ trumps $\ar_1$: namely, the perspective that he currently has. Conversely, if $i$ answers that $\ar_2$ does not trump $\ar_1$, we know that $\ar_2 \nibeatse \ar_1$.%
\footnote{Another way of viewing the relations $\ibeatse$ and $\nibeatse$ goes as follows. Given a perspective $p$, define $\ibeats_p$ as a binary relation over $\allargs$: $\ar_2 \ibeats_p \ar_1$ iff, when $i$ is in the perspective $p$, $\ar_2$ turns $\ar_1$ into an invalid argument. Define $P$ as the set of all possible perspectives. Then, define $\ibeatse = \bigcup_{p \in P} \ibeats_p$, and $\ar_2 \nibeatse \ar_1$ iff $\exists p \in P \suchthat ¬(\ar_2 \ibeats_p \ar_1)$. We favor another presentation because it emphasizes that we consider that we have direct access to $\ibeatse$ and $\nibeatse$, rather than to $\ibeats_p$.}

\begin{remark}
	\label{rq:supportSimple}
	Whereas the two relations $(\ibeatse, \nibeatse)$ allow to capture $i$’s changes of mind about whether a given argument can undermine another argument, the simple support relation $\ileadsto$ adopted here does not permit to capture changes of mind about whether a given argument supports a given proposition. We assume that, in practice, when implementing our approach, propositions will be sufficiently simple and clear, so as to make it safe to assume that $i$ will not change her mind concerning support during the decision process. This is a point to which the analyst will have to pay attention when applying our approach.
If it appears, in real-life implementations, that this assumption is ill-advised, the framework will have to be extended by applying the approach used for $\ibeatse$ to the support relation (this would not raise any specific difficulty). For the time being, in the absence of empirical reasons to believe that the added generality is needed, we choose to use a single $\ileadsto$ relation for simplicity.
\end{remark}

\begin{example}[Ranking (cont.)]
	Consider a set of criteria $J$. Consider the argument $\ar_b = $ “Alternative $\alt_1$ ought to be ranked above $\alt_2$ because $\alt_1$ is better than $\alt_2$ on three criteria while $\alt_2$ is better than $\alt_1$ on only one criterion”, and $\ar_c = $ “It does not make sense to treat all criteria equally in this problem”. Then (depending on $i$’s disposition), it might hold that $\ar_c \ibeatse \ar_b$, and it might hold that $\ar_b \ileadsto t_{\alt_1 \succ \alt_2}$. Note that both may very well hold together.
\end{example}

\begin{definition}[Decision situation]
	We denote a \emph{decision situation} by the tuple $(T, \allargs, {\ileadsto},\allowbreak {\ibeatse}, {\nibeatse})$, with $T, \allargs, {\ileadsto}, {\ibeatse}, {\nibeatse}$ defined as above.
\end{definition}

The part of $i$’s argumentative disposition that remains stable as $i$ changes perspectives is of distinctive interest for decision analysis purposes. Indeed, recall that the emergence of new arguments may lead $i$ to switch perspective. The stable part of her argumentative disposition is therefore a stance that proves resistant to the emergence of new arguments and is, in this sense, argumentatively well-grounded from $i$'s point of view.

Let us therefore define the corresponding stable relations: $\ibeatsst$ is defined as $\ar_2 \ibeatsst \ar_1 ⇔ ¬(\ar_2 \nibeatse \ar_1)$. In plain words, $\ar_2 \ibeatsst \ar_1$ if and only if there is no perspective within which $\ar_2$ does not trump $\ar_1$, or equivalently, $\ar_2 \ibeatsst \ar_1$ if and only if $\ar_2$ trumps $\ar_1$ in all perspectives. Relatedly, $\ar_2 \nibeatsst \ar_1$ is defined as: $\ar_2 \nibeatsst \ar_1 ⇔ ¬(\ar_2 \ibeatse \ar_1)$. 
Hence, $\ar_2 \nibeatsst \ar_1$ indicates that $\ar_2$ never trumps $\ar_1$. This implies, but is not equivalent to, $¬ (\ar_2 \ibeatsst \ar_1)$. 

\begin{example}[Ranking (cont.)]
	Consider alternatives $\alt_1$ and $\alt_2$ such that $\alt_1$ Pareto-dominates $\alt_2$ on criteria $J$. Define $\ar_d$ as an argument that states that $\alt_1$ ought to be ranked above $\alt_2$ because of the Pareto-dominance situation considering criteria in $J$. Then, it might hold that $\ar_d \ileadsto t_{\alt_1 > \alt_2}$. Define $\ar_f$ as “this is an incorrect reasoning because an important aspect to be considered in the problem is fairness and $\alt_1$ is worse than $\alt_2$ in this respect”. Then it might be that $\ar_f \ibeatse \ar_d$ (assuming that $i$ indeed considers fairness as important and that $J$ does not include fairness).
 If $i$ later changes her mind about the importance of fairness, then it will not hold that $\ar_f \ibeatsst \ar_d$. 
\end{example}

This enables us to define a decisive argument as one that is never trumped by any argument in $\allargs$.
\begin{definition}[Decisive argument]
	\label{def:decisiveargument}
	Given a decision situation $(T, \allargs, {\ileadsto}, {\ibeatse},\allowbreak {\nibeatse})$, we say that an argument $\ar \in \allargs$ is \emph{decisive} iff $\forall \ar' \in \allargs$: $\ar' \nibeatsst \ar$.
\end{definition}
 
Notice that decisive arguments can be of very different sorts. Some decisive arguments will be very simple and straightforward arguments, which are so simple that they will be accepted by $i$ whatever the perspective. By contrast, some decisive arguments will be very elaborate ones, taking many aspects of the topic into account and anticipating all sorts of arguments that could trump them, and accordingly never trumped by any other argument.

\begin{example}[Weather forecast]
\label{ex:weather forecast}
Assume that individual $i$ holds that $t$ = “it will rain tomorrow” is supported by the argument $\ar_1$ = “one can expect that it will rain tomorrow because weather forecast predicts so”. (See \cref{fig:weather forecast}.) But imagine that $i$ also holds, at least from some perspective, that $\ar_2$ = “weather forecast is unreliable to infer what the weather will be like tomorrow because weather forecast is often wrong” is a counter-argument that trumps $\ar_1$. Imagine further that $i$ would accept that an argument $\ar_3$ = “although it is often wrong, weather forecast is reliable because it is more often right than wrong” trumps $\ar_2$. Imagine, finally, that no argument trumps $\ar_3$ from any perspective.

In such a case, for $i$, $\ar_1$ is not a decisive argument. However, one can elaborate a more complex argument $\ar$ = “weather forecast predicts that it will rain tomorrow. This may be an incorrect prediction, but weather forecast is more often right than wrong, thus its predictions constitute a sufficient basis to think that it will rain tomorrow”. Notice that $\ar$ includes the reasonings given by $\ar_1$ and $\ar_3$.
Because $\ar$ anticipates that $\ar_2$ could be envisaged to trump it, $\ar$ could be decisive in supporting $t$ (as assumed in \cref{fig:weather forecast}).
\end{example}
\begin{figure}
	\centering
	\begin{tikzpicture}
		\path node (eq) {weather f.\ predicts so ($\ar_1$)};
		\path (eq.east) node[anchor=west] (eql) {$\ileadsto$};
		\path (eql.east) node[anchor=west] (eqr) {rain tomorrow ($t$)};
		\path (eq.south) ++(0, -\BDNodeSep) node[anchor=north] (wages) {weather forecast is often wrong ($\ar_2$)};
		\path (wages) edge[/Beliefs/attack] node[right] {$\ibeatse$} (eq);
		\path (wages.south) ++(0, -\BDNodeSep) node[anchor=north, /Beliefs/decisive] (discrimination) {weather f.\ is more often right ($\ar_3$)};
		\path (discrimination) edge[/Beliefs/attack] node[right] {$\ibeatse$} (wages);
		
		\path (eqr.east) node[anchor=west] (complexl) {\reflectbox{$\ileadsto$}};
		\path (complexl.base east) node[anchor=base west, /Beliefs/decisive] (complex) {complex arg. ($\ar$)};
	\end{tikzpicture}
	\caption{Illustration for \cref{ex:weather forecast,ex:weather forecast2}. The symbol under $\ar_3$ and $\ar$ indicates a decisive argument.}
	\label{fig:weather forecast}
\end{figure}
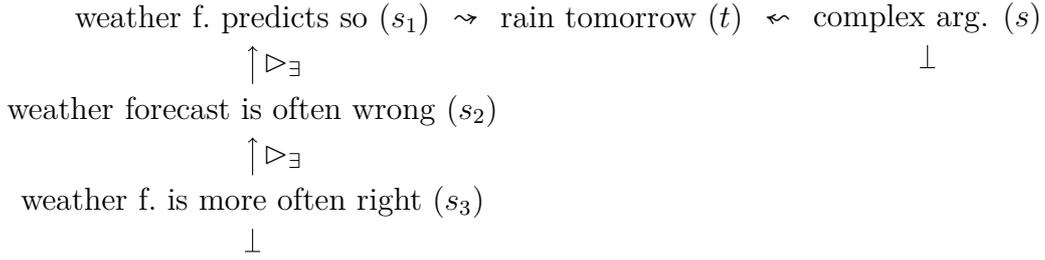

\subsection{Deliberated judgment}
Given a decision situation, we are now in a position to characterize $i$'s stance towards the propositions in $T$ once he has considered all the relevant arguments.
We say that a proposition is justifiable if it is supported by a decisive argument.
A proposition is said to be untenable when each argument supporting it is always trumped by a decisive argument. 

\begin{definition}[Justifiable and untenable propositions]
	\label{def:acceptreject}
	Given a decision situation $(T, \allargs,\allowbreak {\ileadsto},\allowbreak {\ibeatse}, {\nibeatse})$, a proposition $t$ is:
	\begin{itemize}
		\item \emph{justifiable} iff $\exists \ar \in \allargs \suchthat \ar \ileadsto t \text{ and } \forall \ar': \ar' \nibeatsst \ar$;
		\item \emph{untenable} iff $\forall \ar \in \allargs \suchthat \ar \ileadsto t: \exists \ar_c \suchthat \ar_c \ibeatsst \ar \text{ and } \forall \ar_{cc}: \ar_{cc} \nibeatsst \ar_c$.
	\end{itemize}
\end{definition}
Three important aspects of this definition are worth emphazising.

First, we use modal terms to name these notions: we talk about “justifi\emph{able}” rather than “justifi\emph{ed}” propositions. This is because, at a given point of time, individual $i$ might well fail to accept, as a matter of brute empirical fact, a proposition supported by a decisive argument, for example, because she does not know this argument. Similarly, she might accept an untenable proposition. All this is despite the fact that the decisive arguments referred to in the definitions of justifiable and untenable propositions are decisive \emph{according to $i$'s argumentative disposition} – that is, by $i$'s own standards.

Second, notice that, according to our definition, a proposition can’t be both justifiable and untenable, but it may be neither justifiable nor untenable. This may be the case if all the arguments supporting $t$ have counter-arguments, but at least one argument supporting $t$ has no decisive counter-argument.

Lastly, according to our definition, it is possible for a proposition $t$ to be justifiable and for not-$t$, or more generally for any proposition $t'$ in logical contradiction with $t$ or having empirical incompatibilities with $t$, to be justifiable too.
This specific definition allows to encompass situations in which there are intrinsically no more reason to accept $t$ than $t'$. This can happen even when it is clear and evident for $i$ that $t$ and $t'$ are incompatible, and even in situations where this incompatibility between $t$ and $t'$ is highlighted in some argument examined during the decision process.%
\footnote{Relatedly, notice that there is an important asymmetry between the notions of justifiable and untenable. Because $t$ and some incompatible $t'$ can both be justifiable, the fact that $t$ is justifiable does not necessarily imply that the fate of $t$ in $i$'s view is entirely settled by its justifiability. By contrast, there is no way an untenable proposition could come back into the scene.}
This is a consequence of our definition of the trump relation, and it reflects the important idea that, as a matter of fact, in some decision situations, even if one takes all the relevant arguments into account, it can happen that several, mutually incompatible propositions are equally supported. It is part of the very aim of decision-aid, in such situations, to unveil the fact that mutually incompatible propositions are equally supported.%
\footnote{Somewhat similar distinctions are discussed in formal argumentation theory about skeptical versus credulous justification \citep{prakken_combining_2006}. Delving into the details of a comparative analysis falls beyond the scope of the present article.}%

Decision situations allowing to classify unambiguously all propositions in the agenda into justifiable or untenable propositions are of distinctive interest. Let us term such decision situations “clear-cut”.

\begin{definition}[Clear-cut situation]
	A decision situation $(T, \allargs, {\ileadsto}, {\ibeatse}, {\nibeatse})$ is \emph{clear-cut} iff each proposition in $T$ is either justifiable or untenable.
\end{definition}

Given a decision situation, we can now define $i$'s \ac{DJ} as those propositions $t \in T$ that are justifiable. \begin{definition}[\ac{DJ} of $i$]
\label{def:justifiable}
	The \acl{DJ} corresponding to a decision situation $(T, \allargs, {\ileadsto}, {\ibeatse}, {\nibeatse})$ is:
	\begin{equation}
		\label{eq:DJ-u}
		T_i = \set{t \in T \suchthat t \text{ is justifiable}}.
	\end{equation}
\end{definition}

This notion of \ac{DJ}, as we define it, captures what we take to be an important idea underlying \citeauthor{goodman_fact_1983}’s \citeyearpar{goodman_fact_1983} and \citeauthor{rawls_theory_1999}’ \citeyearpar{rawls_theory_1999} concept of “reflective equilibrium”. This idea is that, if $i$ manages, through an iterative process of revision of her opinion through the integration of new elements or arguments, to reach an “equilibrium” which is stable with respect to the integration of new elements, then the opinion reached at “equilibrium” is of distinctive interest -- it captures $i$'s “well-considered” or “true” opinion in some sense.%
\footnote{That said, our notion of \ac{DJ} does not claim to reflect faithfully all the aspects of the notion of “reflective equilibrium” as used by the authors mentioned above. A thorough exploration of the links between our formal framework and these philosophical theories falls beyond the scope of the present article.}

Notice that the meaning of this definition depends on the interpretation given to $\allargs$ (see the beginning of \cref{sec:core}).
 In the idealistic interpretation, $i$'s \ac{DJ} is unique and fixed once and for all. In the pragmatic interpretation, $i$'s \ac{DJ} may evolve over time, as new arguments emerge.

\begin{example}[Weather forecast (cont.)]
\label{ex:weather forecast2}
To explain clearly this definition, it is useful to come back to our previous example (\cref{fig:weather forecast}) of individual $i$ who holds that “weather forecast is often wrong” ($\ar_2$) is a counter-argument that trumps “it will rain tomorrow because weather forecast predicts so” ($\ar_1$). We have seen that a more complex argument ($\ar$), including both “weather forecast predicts that it will rain tomorrow” and an additional sub-argument that trumps $\ar_2$, can turn out to be a decisive argument to support “it will rain tomorrow” ($t$). In such a case, $t$ belongs to $i$'s deliberated judgment, despite the fact that he might claim otherwise if not confronted with the complex argument above.
\end{example}

\begin{example}[Weather forecast (variant)]
\label{ex:weather forecast contr}
	In this example $T$ contains two propositions: $t_1$ is the proposition according to which it will rain tomorrow, and $t_2$ is the contrary proposition. Two corresponding arguments are $\ar_1$ and $\ar_2$ -- two weather forecasts from different sources that predict respectively that it will rain and that it will not. Assuming that $i$ attributes equal credibility to both sources and considers no other argument to be relevant, he might end up with both $t_1$ and $t_2$ in his deliberated judgment. This should not be interpreted as meaning that $i$ is incoherent, but simply as a situation where different propositions are equally justified for lack of means to tell them apart. Similarly, scientists can consider two contradictory hypotheses plausible, for lack of current knowledge; or someone may hold that two incompatible acts are equally (im)moral.
\end{example}

\section{Issues of empirical validation}
\label{sec:empirical}
The former section clarified definitions and explained the articulations between the key concepts of our framework, at a rather abstract level. Now we want to investigate how this framework can be confronted with empirical reality. For that purpose, we will examine how one can test a \emph{model} of the support and trump relations built by a decision analyst trying to capture the deliberated judgment of a decision-maker.

Let us define a model $\eta$ of a decision situation as a pair of relations $\mleadsto \subseteq \allargs × T$ and $\mbeats \subseteq \allargs × \allargs$. These relations are not necessarily an approximation of the real $\ileadsto, \ibeatse$ relations characterizing $i$. Indeed, the chief aim of the model is to know $i$’s \ac{DJ}, not to reflect in detail what $i$ thinks about all arguments, which would arguably not be achievable (we will come back to this important point below).

Define $T_\eta$ as the set of propositions that the model $\eta$ claims are supported:
\begin{equation}
	T_\eta = {\mleadsto}(\allargs) = \set{t \in T \suchthat \exists \ar \in \allargs \suchthat \ar \mleadsto t}.
\end{equation}

\begin{example}[Ranking (cont.)]
\label{ex:mavt}
\newcommand{\altc}{c}
\newcommand{\altl}{l}
\newcommand{\altp}{p}

We have already defined a set of alternatives $\allalts$, propositions $T$ representing possible comparisons of the alternatives, and criteria $J$. Consider further a set of criteria functions $(g_j)_{j \in J}$ evaluating all the alternatives $\alt \in \allalts$ using real numbers: $g_j: \allalts → \R$.

Imagine that $i$’s problem is to decide which kind of vegetable to grow in his backyard. Assume an analyst providing decision-aid to $i$ considers that the problem can be reduced to a ranking between three candidates: carrots, lettuce and pumpkins, denoted by $c, l, p \in \allalts$. The analyst believes that $i$ is ready to rank vegetables according to exactly two criteria. The analyst has obtained six real numbers $g_j(\alt)$, representing the performances of each alternative on each criteria, and believes that $i$ is ready to rank vegetables according to the sum of their performances on the two criteria, $v(\alt) = g_1(\alt) + g_2(\alt)$.

The analyst can now try to represent $i$'s attitude using a model $\eta = \left(\mleadsto, \mbeats\right)$ by producing sentences that explain to $i$ the “reasoning” underlying the definition of $v$. Assume the values given by $v$ position carrots as winners. The analyst could define an argument $\ar_{(\altc, \altl)}$ “carrots are a better choice than lettuce because carrots score $g_1(\altc)$ on criterion one, and $g_2(\altc)$ on criterion two, which gives it a value $v(\altc)$, whereas lettuce scores $g_1(\altl)$ on criterion one, and $g_2(\altl)$ on criterion two, which gives it an inferior value $v(\altl)$”. In the model of the analyst, this argument supports the proposition that carrots are ranked higher than lettuce: $\ar_{(\altc, \altl)} \mleadsto t_{\altc \succ \altl}$. The
model contains similar arguments in favor of other propositions $t \in T$ that are in agreement with the values given by $v$. In our example, the analyst furthermore believes that no counter-arguments are necessary and thus defines $\mbeats = \emptyset$.
\end{example}

\subsection{Validity and the problem of observability}
\label{sec:valObs}
Because the point of carving out $\eta$ is to capture i's \acl{DJ} $T_i$, we can define a valid model as one that correctly captures $T_i$.

\begin{definition}[Validity]
\label{valid}
	A model $\eta$ is valid iff $T_\eta=T_i$.
\end{definition}

How can the analyst determine if a given model $\eta$ is a valid one?

Let us assume that the only information that he can use for that purpose is the one he can get by querying $i$ – and is, in that sense, “observable” for him. \acp{DJ} are not observable in that sense. Indeed, $i$'s \ac{DJ} are defined in terms of $\nibeatsst$. But observing $\nibeatsst$ would require that $i$ takes successively all the possible perspectives she can have, which is unrealistic.%
\footnote{This would amount to assume that $i$ already knows all the arguments and can aggregate them successfully. If this were possible, $i$ would probably not need help from an analyst.}

In the remainder of this section, we explain how we handle this conundrum in two steps. First, \cref{sec:modelsEasy} introduces a provisional solution, by identifying conditions that guarantee the existence of a model allowing to identify $i$'s \ac{DJ} on the basis of what we will call an  “operational” validity criterion – that is, a criterion based on observable data. Then, \cref{sec:weakening} explores how these conditions can be weakened.

\subsection{Existence of a valid model and its conditions}
\label{sec:modelsEasy}
In this subsection, we introduce apparently reasonable conditions about the way $i$ reasons and about the decision situation. Our theorem will then guarantee that a model exists and captures correctly $i$’s \ac{DJ} if those conditions are satisfied on $\allargs$ and if the model satisfies a validity criterion that, as opposed to validity itself, can be directly checked on the basis of observable data (an ``operational validity'' criterion).

\subsubsection{Conditions}
A first condition about $\ibeatse$ mandates a certain form of stability. It assumes that $i$ possibly changes her mind about whether an argument $\ar'$ trumps another one only when there exists another argument that trumps $\ar'$.

\begin{condition}[Answerability]
	\label{def:justifiableStrong}
	A decision situation $(T, \allargs, {\ileadsto}, {\ibeatse}, {\nibeatse})$ satisfies \emph{Answerability} iff, for all pairs of arguments $(\ar, \ar')$:
	\begin{equation}
		\ar' \ibeatse \ar \text{ and } \ar' \nibeatse \ar ⇒ \exists \ar_c \suchthat \ar_c \ibeatse \ar'.
	\end{equation}
\end{condition}

Let us now turn to the second condition. It has to do with the way $i$ reasons. Imagine that $i$ finds himself in the following uneasy situation. He declares that $\ar_1$ is trumped by $\ar_2$. However, $i$ is also ready to declare that $\ar_2$ is in turn trumped by $\ar_3$, a decisive argument. In such a situation, it seems natural enough to assume that, if we carve out an argument $\ar$, playing the same argumentative role as $\ar_1$, but anticipating and defeating attempts to trump it using $\ar_2$, $i$ will endorse $\ar$.

This assumption is formalized by the condition \emph{Closed under reinstatement} below. To write it down, we first need to formalize, thanks to the following notion of \emph{replacement}, the idea that a set of arguments is at least as powerful as another argument, from the point of view of its argumentative role. We say a set of arguments $\args \subseteq \allargs$ replaces an argument $\ar \in \allargs$ whenever all the arguments trumped by $\ar$ are also trumped by some argument $\ar' \in \args$, and all the propositions supported by $\ar$ are also supported by some argument $\ar' \in \args$.%
\footnote{Note that the replacer may be more powerful than the argument it replaces, in the sense that it may trump arguments or support propositions than the replaced argument did not trump or support.}
\begin{definition}[Replacing arguments]
	\label{def:replacement}
	A set of arguments $\args \subseteq \allargs$ \emph{replaces} $\ar \in \allargs$ iff ${\ibeatse}(\ar) \subseteq {\ibeatse}(\args)$ and ${\ileadsto}(\ar) \subseteq {\ileadsto}(\args)$. 
	We say that $\ar'$ replaces $\ar$, with $\ar, \ar' \in \allargs$, to mean that $\{\ar'\}$ replaces $\ar$.
\end{definition}
	
\begin{condition}[Closed under reinstatement]
	\label{def:closed}
	A decision situation $(T, \allargs,\allowbreak {\ileadsto},\allowbreak {\ibeatse},\allowbreak {\nibeatse})$ is \emph{closed under reinstatement} iff, $\forall \ar_1 ≠ \ar_2 ≠ \ar_3 ≠ \ar_1 \in \allargs$ such that $\ar_3 \ibeatsst \ar_2 \ibeatse \ar_1$, with $\ar_3$ decisive:
	\begin{equation}
		\exists \ar \suchthat \ar \text{ replaces } \ar_1 \text{ and } {\ibeatseinv}(\ar) \subseteq \ibeatseinv(\ar_1) \setminus \{\ar_2\}.
	\end{equation}
\end{condition}
The condition mandates that, whenever some decisive argument always trumps $\ar_2$, which in turn trumps $\ar_1$, it is possible to replace $\ar_1$ by an argument that is no longer trumped by $\ar_2$ and is not trumped by any other argument than those trumping $\ar_1$.%
\footnote{Such a configuration of arguments, where $\ar_3$ trumps $\ar_2$ which in turns trumps $\ar_1$, recalls the notion of “strong defense” in argumentation theory \citep{baroni_principle-based_2007}. A further discussion of this issue falls beyond the scope of this paper.}

Finally, we introduce two conditions on the size of the relation $\ibeatse$.

Let us call a chain of length $k$ in $\ibeatse$ a finite sequence $\ar_i$ of arguments in $\allargs$, $1 ≤ i ≤ k$, such that $\ar_i \ibeatse \ar_{i + 1}$ for $1 ≤ i ≤ k - 1$. An infinite chain is an infinite sequence $\ar_i$ such that $\ar_i \ibeatse \ar_{i + 1}$ for all $i \in \N$.

\begin{condition}[Bounded width]
\label{def:B.br}	
A decision situation $(T, \allargs, {\ileadsto}, {\ibeatse}, {\nibeatse})$ has a \emph{bounded width} iff there is no argument that is trumped by an infinite number of counter-arguments.
\end{condition}

\begin{condition}[Bounded length]
\label{def:B.lg}
	A decision situation $(T, \allargs, {\ileadsto}, {\ibeatse}, {\nibeatse})$ has a \emph{bounded length} iff there is no infinite chain in $\ibeatse$. (Cycles in $\ibeatse$ are therefore excluded as well.)
\end{condition}

\subsubsection{Operational validity criterion}
Let us now define the following “operational” validity criterion for a model $\eta$ intended to capture $i$'s \ac{DJ}. We term it “operational” to emphasize that, as opposed to the definition of validity (\cref{valid}), it can be checked on the sole basis of observable data.

\begin{definition}[Operational validity criterion]
	\label{def:validity}
	A model $\eta$ of a decision situation is operationally valid iff, whenever $(\ar \mleadsto t)$, it holds that $[\ar \ileadsto t]$ and $[\forall \ar_c \in \allargs: (\ar_c \nibeatse \ar) ∨ (\exists \ar_{cc} \mbeats \ar_c ∧ \ar_{cc} \ibeatse \ar_c)]$, and whenever $t$ is not supported by $\eta$, $\forall \ar \ileadsto t: \exists \ar_c \mbeats \ar ∧ \ar_c \ibeatse \ar$.
\end{definition}

This criterion amounts to partially comparing, on the one hand, $i$'s argumentative disposition towards propositions and arguments and, on the other hand, $\eta$'s representations of $i$’s argumentative disposition.%
\footnote{This procedure could be considered as a persuasion dialogue \citep{prakken_models_2009}.}
More precisely, a model satisfies the operational validity criterion (for short: is operationally valid) iff:
\begin{enumerate}[label=({\roman*}), ref={\roman*}]
	\item arguments that, according to the model, support a proposition $t$ are indeed considered by $i$ to support $t$;
	\item whenever a model uses an argument $\ar$ to support a proposition, and that argument is trumped by a counter-argument $\ar_c$, the model can answer with a counter-counter-argument, using a counter-counter-argument that $i$ confirms indeed trumps the counter-argument $\ar_c$;
	\item whenever an argument $\ar$ supports a proposition that the model does not consider to be supported, the model is able to counter that argument using a counter-argument that $i$ confirms indeed trumps $\ar$.
\end{enumerate}

As required, this criterion is uniquely based on observable data. Indeed, recall that the only observable data that the analyst can use are the ones obtained by querying $i$ by asking her if a given argument $\ar_2$ trumps another argument $\ar_1$. If she replies that it does, this is enough to conclude that, according to her, $\ar_2 \ibeatse \ar_1$. Indeed, in such a case, we know that there is at least one perspective within which she thinks that $\ar_2$ trumps $\ar_1$: namely, the perspective that she currently has. Querying $i$ can thus provide the information needed to check if a model is operationaly valid.

\subsubsection{Theorem}
Because querying $i$ will not give enough information to know that $\ar_2 \ibeatsst \ar_1$ (if indeed $\ar_2 \ibeatsst \ar_1$), querying $i$ will never allow to directly claim that a model satisfies the definition of validity (\cref{valid}). What we need therefore is a means to ensure that an operationally valid model is a valid one. This is provided by the following theorem.
\begin{theorem}
	\label{thm:clearcutWeak}
	Assume a decision situation $(T, \allargs, {\ileadsto}, {\ibeatse}, {\nibeatse})$ is Closed under reinstatement, Answerable and has Bounded length and width. Then: i) the decision situation is clear-cut; ii) there exists an operationally valid model of that decision situation; iii) any operationally valid model $\eta$ satisfies $T_i = T_\eta$.
\end{theorem}
\Cref{thm:clearcutStrong} 
(in \cref{sec:weakening}) generalizes this theorem. It is proven in \cref{sec:proofs}.

\begin{example}[Budget reform]
\label{ex:budget}
Let us take a non trivial example that will be used to illustrate how \cref{thm:clearcutWeak} can be used and why we need to go beyond this first theorem. Imagine that $i$ is a political decision-maker. She wants to run for an election, and is elaborating her policy agenda. She has heard about  \citeauthor{meinard_measuring_2017}'s \citeyearpar{meinard_measuring_2017} (thereafter referred to as “M”) argument that, according to a popular survey, biodiversity should be ranked after retirement schemes and public transportation, but before relations with foreign countries, order and security, and culture and leisure in the expenses of the State. Assume that $i$ wants to make up her mind about the single proposition $t =$ “I should include in my agenda a reform to increase public spending on biodiversity conservation so as to rank biodiversity higher than relations with foreign countries in the State budget”.

She requests the help of a decision analyst. The latter starts by reviewing the literature to identify a set of arguments with which he will work. (The arguments are illustrated in \cref{fig:budget}.) He thereby identifies that proposition $t$ can be considered to be supported by $\ar =$ “M's finding (stated above) is based on a large scale survey and quantitative statistical analysis, and their protocol was designed to track the preferences that citizens express in popular votes. There are therefore scientific reasons to think that a policy package including the corresponding reform will gather support among voters.” Pursuing his exploration of the recent economic literature on environmental valuation methods, the analyst could identify only two counter-arguments to $\ar$:
\begin{itemize}
	\item $\ar_{c1}=$ “M's measure is extremely rough as compared to more classical economic valuations, such as contingent valuations and the like \citep{kontoleon_biodiversity_2007}, which makes it non credible as a guide for policy”;
	\item $\ar_{c2}=$ “M claim to value biodiversity \emph{per se}. The very meaning of such an endeavor is questionable because it is too abstract. More classical economic valuations are focused on concrete objects and projects, which is more promising”.
\end{itemize}

But he also found a counter-counter-argument to each of these counter-arguments:
\begin{itemize}
	\item $\ar_{c1c}=$ “Biodiversity is not the kind of thing about which people make decisions in their everyday life. Their preferences about it are accordingly likely to be rough. The exceedingly precise measurements provided by contingent valuations and the like are therefore more a weakness than a strength”;
	\item $\ar_{c2c}=$ “Abstract notions such as biodiversity are an important determining factor for many people when they make decisions. Eschewing to value them is ill-founded”.
\end{itemize}

Imagine further that the analyst has not found any argument liable to trump either $\ar_{c1c}$ or $\ar_{c2c}$.

\begin{figure}
	\centering
	\begin{tikzpicture}
		\path node (a) {will gather support ($\ar$)};
		\path (a.base east) node[anchor=base west] (al) {$\mleadsto$};
		\path (al.base east) node[anchor=base west] (ar) {increase spendings ($t$)};
		\path (a.south) ++(0, -\BDNodeSep) node[anchor=north east] (ac1) {rough measure ($\ar_{c1}$)};
		\path (a.south) ++(0, -\BDNodeSep) node[anchor=north west] (ac2) {abstract ($\ar_{c2}$)};
		\path (ac1) edge[/Beliefs/attack] node[right, xshift=0.8ex, yshift=-0.5ex] {$\mbeats$} (a);
		\path (ac2) edge[/Beliefs/attack] node[right] {$\mbeats$} (a);
		\path (ac1.south) ++(0, -\BDNodeSep) node[anchor=north, /Beliefs/decisive] (ac1c) {inherent ($\ar_{c1c}$)};
		\path (ac1c) edge[/Beliefs/attack] node[right] {$\mbeats$} (ac1);
		\path (ac2.south) ++(0, -\BDNodeSep) node[anchor=north, /Beliefs/decisive] (ac2c) {important ($\ar_{c2c}$)};
		\path (ac2c) edge[/Beliefs/attack] node[right] {$\mbeats$} (ac2);
		
	\end{tikzpicture}
	\caption{Illustration for \cref{ex:budget}. (Only the arguments used by the model $\eta$ are displayed.)}
	\label{fig:budget}
\end{figure}
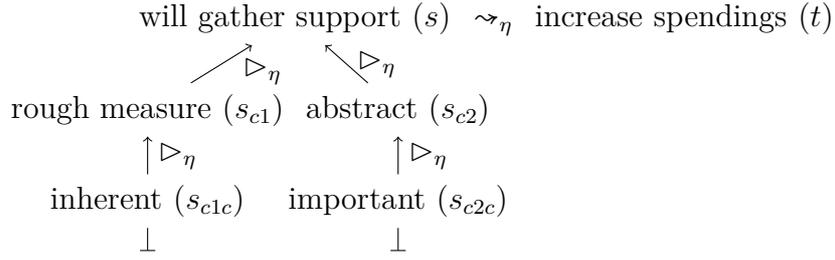

Define $\ar_{1, \text{reinstated}}$ as: “[\emph{content of $\ar$}]; this is a rough measure but [\emph{content of $\ar_{c1c}$}]”; similarly, define $\ar_{2, \text{reinstated}}$ as “[\emph{content of $\ar$}]; the very meaning could be questioned because it is highly abstract, but [\emph{content of $\ar_{c2c}$}]”; and define $\ar_\text{reinstated}$ as “[\emph{content of $\ar$}]; this is a rough measure but [\emph{content of $\ar_{c1c}$}]; the very meaning could be questioned because it is highly abstract, but [\emph{content of $\ar_{c2c}$}]”. Define $\args \subseteq \allargs$ as the set of argument comprising $\ar$, $\ar_{c1}$, $\ar_{c2}$, $\ar_{c1c}$, $\ar_{c2c}$, $\ar_{1, \text{reinstated}}$, $\ar_{2, \text{reinstated}}$ and $\ar_\text{reinstated}$.

Assume that the analyst is justified to think that $i$'s reasoning is such that $\allargs$ satisfies Closed under reinstatement, Answerability, Bounded length and Bounded width. Recall now that, in order to identify the propositions lying in $T_i$, the analyst must identify arguments supporting propositions in $T_i$, such that these arguments can resist counter-arguments from the whole of $\allargs$. In other words, the analyst must test the claims of the model not only against the counter-arguments in $\args$, but against the whole of $\allargs$, which the analyst ignores.

Imagine now that the analyst assumes that, even though $\args$ is a strict subset of $\allargs$, $\args$ is a good enough approximation of $\allargs$, in the sense that there is no argument in $\allargs \setminus \args$ that trumps any argument in $\args$ or that supports $t$.
Thanks to \cref{thm:clearcutWeak}, the analyst can then deduce that the situation is clear-cut and that there exists a valid model of the decision situation.

The next step for him is to carve out a model $\eta$ reproducing the relations between arguments that he found in the literature, and then to test whether his model is operationally valid using \cref{def:validity}. 
In order to validate $\eta$, he would first ask $i$ whether she agrees that $\ar$ supports $t$. If so, he then would check whether $i$ considers that $\ar_{c1}$ is a counter-argument to $\ar$, in which case the analyst would check that the counter-counter-argument that he envisaged, $\ar_{c1c}$, is considered by $i$ to trump $\ar_{c1}$. The analyst would then proceed in a similar way with the second chain of counter-arguments ($\ar_{c2}$ and $\ar_{c2c}$), and verify that, as $\eta$ hypothesizes, $i$ does not take any other argument in $\args$ to trump $\ar$. This would, eventually, allow him to conclude on the validity of the model $\eta$.
Should it prove operationally valid, the analyst could then conclude that $T_i=\{t\}$ (using \cref{thm:clearcutWeak} and $T_\eta=\{t\}$). 

But notice that this whole story only works because we assumed that arguments in $\allargs \setminus \args$ never trump any argument in $\args$. This assumption is clearly unrealistic: any slight reformulation of $\ar_{c1}$, for example, will most likely also trump $\ar$. This is not the only unrealistic assumption in our hypothetical scenario: it is also unlikely that the whole set $\allargs$ indeed satisfies Bounded length, for example. This condition requires an absence of cycle in the trump relation. While this may be considered to hold on $\args$, it is possible that some ambiguous or poorly phrased arguments in $\allargs$ would confuse $i$ in such a way that $i$ will declare, for example, that $\ar_1 \ibeatse \ar_2 \ibeatse \ar_3 \ibeatse \ar_1$ for some triple of such unclear arguments. Hence the need to go beyond \cref{thm:clearcutWeak}.
\end{example}

\Cref{thm:clearcutWeak} embodies an important step towards being able to confront models of deliberated judgment with empirical reality, by spelling out sufficient conditions upon which unrolling the procedures of refutation is not a pure waste of time and energy, because there is something to be found. It also illustrates the potential usefulness of the notion of operational validity. Indeed, since the point of the modeling endeavor in our context is to capture $T_i$, we know by virtue of iii) in \cref{thm:clearcutWeak} that, if the corresponding conditions are met, and if we have good reasons to believe that we have an operationally valid model, then we can admit that it captures $T_i$.

However, establishing this theorem cannot be more than just a first step. As illustrated in \cref{ex:budget}, the conditions above are quite heroic. One cannot realistically expect that real-life decision situations will fulfill these conditions. The most important issue is that we need a means to distinguish $\allargs$ from the restricted set of arguments with which the analyst works in practice. And we need means to make sure that the restricted set indeed “covers” the matter “sufficiently”, so as to escape the situation in which the analyst is locked in \cref{ex:budget}, where he finds himself condemned to make wildly unrealistic assumptions. The next subsection tackles this pivotal issue.

\subsection{Weakening of some conditions}
\label{sec:weakening}
To obtain the results we want, all we actually need is that it should be possible to define a subset of arguments $\clargs \subseteq \allargs$ that satisfies conditions akin to the ones defined above, and which are sufficient to cover the topic at hand.

Let us start by formalizing the requirement, for $\clargs$, to cover the topic at hand. What we want is that all the arguments needed for the decision-maker to make up her mind about the topic should be encapsulated in $\clargs$. This means that, if arguments from $\ar \in \allargs \setminus \clargs$ are brought to bear, it should be possible either to discard them or to show that they can be replaced by arguments in $\clargs$.

This is done thanks to the following formal definitions and condition.

\begin{definition}[Unnecessary argument]
	\label{def:unnecessary}
	Given a decision situation and a subset $\clargs \subseteq \allargs$ of arguments, we say that $\args \subseteq \allargs$ \emph{essentially replaces} $\ar \in \allargs$ iff $(\ibeatse(\ar) ∩ \clargs) \subseteq {\ibeatse}(\args)$ and ${\ileadsto}(\ar) \subseteq {\ileadsto}(\args)$. 
	
	Let $\argscldec = \clargs ∩ \overline{\ibeatse(\allargs)}$ denote the decisive arguments in $\clargs$.
	We say that an argument $\ar \in \allargs$ is \emph{resistant} iff it is not trumped by any argument in $\argscldec$. Let $\argsclres = \clargs ∩ \overline{\ibeatse(\argscldec)}$ denote the resistant arguments in $\clargs$.
	
	We say that an argument $\ar \in \allargs$ is \emph{unnecessary} iff $\ar$ is trumped by a resistant argument from $\clargs$ or $\ar$ is essentially replaceable by $\argsclres$.
	In formal terms: $\ar \in \ibeatse(\argsclres)$ or $[(\ibeatse(\ar) ∩ \clargs) \subseteq {\ibeatse}(\argsclres)$ and ${\ileadsto}(\ar) \subseteq {\ileadsto}(\argsclres)]$.
\end{definition}

\begin{condition}[Covering set of arguments]
\label{def:cover}
Given a decision situation and a set of arguments $\clargs \subseteq \allargs$, $\clargs$ is \emph{covering} iff all arguments $\ar \in \allargs \setminus \clargs$ are unnecessary.
\end{condition}

Let us now relax the conditions of \cref{thm:clearcutWeak} by formulating weaker requirements confined to $\clargs$. This adaptation is straightforward for \cref{def:justifiableStrong,def:closed}.

\begin{condition}[Set of arguments allowing answerability]
	\label{def:justUnstSet}
Given a decision situation and a subset $\clargs \subseteq \allargs$ of arguments, we say that the set $\clargs$ satisfies \emph{Answerability} iff, for all $\ar \in \allargs, \ar' \in \clargs$: $\ar' \ibeatse \ar \text{ and } \ar' \nibeatse \ar ⇒ \exists \ar_c \in \allargs \suchthat \ar_c \ibeatse \ar'$.
\end{condition}

\begin{condition}[Set of arguments closed under reinstatement]
	\label{def:closedSet}
	Given a decision situation $(T, \allargs, {\ileadsto}, {\ibeatse}, {\nibeatse})$ and a subset $\clargs \subseteq \allargs$ of arguments, we say that the set $\clargs$ is \emph{closed under reinstatement} iff, $\forall \ar_1, \ar_3 \in \clargs, \ar_1 ≠ \ar_3, \ar_3$ not trumping $\ar_1$, $\ar_3$ decisive:
	\begin{equation}
		\exists \ar \in \clargs \suchthat \ar \text{ replaces } \ar_1 \text{ and } \ibeatseinv(\ar) \subseteq \ibeatseinv(\ar_1) \setminus \ibeatsst(\ar_3).
	\end{equation}
\end{condition}
This condition is vacuous when there is no $\ar_2$ such that $\ar_3 \ibeatse \ar_2 \ibeatse \ar_1$: in that case, $\ar_1$ replaces itself.

Similarly, we can relax \cref{def:B.br} and apply it to a subset of arguments. When an argument has very numerous counter-arguments, one may think that their vast number might spring from some common reasoning that they share. For example, an argument might involve some real value as part of its reasoning, and be multiplied as infinitely many similar arguments of the same kind using tiny variations of that real value. If so, and if we know that we can convincingly rebut each of these counter-arguments, we might believe that only a small number of counter-counter-arguments will suffice to rebut the counter-arguments.

\begin{definition}[Defense]
	We say $\ar \in \allargs$ is \emph{$\clargs$-defended} iff all the arguments $\ar_c$ trumping $\ar$ are trumped by a decisive argument in $\clargs$, or formally, $\forall \ar_c \in \allargs \suchthat \ar_c \ibeatse \ar: (\exists \ar_{cc} \in \clargs \suchthat \ar_{cc} \ibeatsst \ar_c, \ar_{cc} \text{ decisive})$.
We say $\ar \in \allargs$ is \emph{$(j, \clargs)$-defended} iff there exists a set $\args \subseteq \clargs$ of arguments of cardinality at most $j$ such that $\ar$ is $\args$-defended (thus, if $j$ arguments from $\clargs$ suffice to defend $\ar$).
\end{definition}

\begin{condition}[Set of arguments with width bounded by $j$]
	\label{def:setB.b}
	Given a decision situation and a natural number $j$, a set of arguments $\clargs \subseteq \allargs$ has \emph{width bounded by $j$} iff, for each argument $\ar \in \clargs$, if $\ar$ is \emph{$\clargs$-defended}, then it is $(j, \clargs)$-defended.
\end{condition}
The condition is vacuously true when no argument in $\allargs$ is trumped by more than $j$ counter-arguments.

Our last condition relaxes \cref{def:B.lg}. We want to exclude \emph{some} of the long chains in $\allargs$. But we want to tolerate long chains, including cycles, among unclear arguments. Indeed, anecdotal evidence from ordinary argumentation situations suggests that in many (otherwise interesting) decision situations, cycles do appear in trump relations among arguments (for example, because arguments can use ambiguous terms). However, this does not necessarily prevent the situation from being modelizable in our sense. What we do need is to avoid some of the cycles or chains that involve “too many” arguments from $\clargs$, in a somewhat technical sense captured by the following condition.

\begin{condition}[Set of arguments with length bounded by $k$]
	\label{def:setB.lg}
	Given a decision situation, a natural number $k$, and a set of arguments $\clargs$, define a binary relation $Q$ over $\clargs$ as $\ar_2 Q \ar_1$ iff $\ar_2 \ibeatse \ar_1$ or $\ar_2 \ibeatse \ar \ibeatse \ar_1$ for some $\ar \in \allargs$, thus, $Q = (\ibeatse ∪ (\ibeatse \circ \ibeatse)) ∩ (\clargs × \clargs)$. Let $Q^1 = Q$ and $Q^{k+1} = Q^k \circ Q$ for any natural number $k$. The set $\clargs$ has length bounded by $k$ iff $\nexists \ar_2, \ar_1 \in \clargs \suchthat \ar_2 Q^{k+1} \ar_1$, thus, iff it is impossible to reach an argument from $\clargs$, starting from an argument from $\clargs$, following $Q$ more than $k$ times.
\end{condition}

This condition tolerates cycles\footnote{Cycles in our sense have to be distinguished from cycles involving an attack relation as defined in formal argumentation theory. We do not deny that cycles of attacks in the formal argumentation sense often happen, and \cref{def:setB.lg} does not exclude cycles understood in that sense: these cycles are generally not cycles in “trump” relations. We consider that an argument $\ar_2$ trumps another one only when $i$ considers that the first one is strong enough to render the second one ineffective.
This definition relies on an asymmetry, $\ar_2$ being, in a sense, “favored over” $\ar_1$. Our trump relation is therefore somewhat analogical to a strict preference relation, for which an assumption of acyclicity is commonplace in the literature.} in $\ibeatse$ that involve only arguments picked outside the chosen set $\clargs$. It only forbids a subset of the situations where a cycle (or a too long chain) is built that involve arguments from $\clargs$. For example, it excludes a situation where $\ar_2 \ibeatse \ar \ibeatse \ar_1 \ibeatse \ar_2$ for some $\ar_1, \ar_2 \in \clargs$ and $\ar \notin \clargs$.\footnote{Readers used to decision theoretic axiomatizations might find this condition odd, since axioms usually mandate conditions considered more “basic”, such as transitivity and irreflexivity, and derive from them the conclusion that cycles are forbidden. This strategy does not work for our setting (or is not applicable in a simple way), because “basic” conditions such as transitivity would be unreasonable to impose here. For example, given $\ar_3 \ibeatse \ar_2$ and $\ar_2 \ibeatse \ar_1$, it is easy to think about situations where $i$ would consider that $\ar_3 \nibeatsst \ar_1$, and to think about situations where $i$ would consider that $\ar_3 \ibeatse \ar_1$. Neither anti-transitivity nor transitivity can thus be reasonably imposed (and our current condition avoids such requirements). Studying which conditions exactly are necessary to ban cycles (or make them innocuous) in our setting would be interesting, but it does not seem crucial at this stage. Indeed, in concrete settings we consider that cycles involving arguments from $\clargs$ are unlikely to occur. (This claim should be backed up by empirical studies.)}

Thanks to \cref{def:cover,def:justUnstSet,def:closedSet,def:setB.b,def:setB.lg}, we are now in a position to define our set of arguments of interest.

\begin{definition}[CAC arguments]
	Given a decision situation and a set $\clargs \subseteq \allargs$, we say that $\clargs$ is \emph{clear} and \emph{covering}, or CAC, iff it is Closed under reinstatement and Answerable, and has width bounded by some number $j$ and length bounded by some number $k$, and is such that all arguments $\ar \in \allargs \setminus \clargs$ are unnecessary.
\end{definition}

Following the same rationale, we can define an operational criterion echoing \cref{def:validity}.

\begin{definition}[$\clargs$-operational validity]
	Given a decision situation and a set $\clargs \subseteq \allargs$, we define a model $\eta$ as \emph{$\clargs$-operationally valid} iff for all $(\ar \mleadsto t)$, $\ar \in \allargs$, we have $[\ar \ileadsto t]$ and $[\forall \ar_c \in \clargs: (\ar_c \nibeatse \ar) ∨ (\exists \ar_{cc} \in \allargs \suchthat \ar_{cc} \mbeats \ar_c ∧ \ar_{cc} \ibeatse \ar_c)]$, and when $t$ is not supported by $\eta$, $\forall \ar \in \clargs \suchthat \ar \ileadsto t: (\exists \ar_c \in \allargs \suchthat \ar_c \mbeats \ar ∧ \ar_c \ibeatse \ar)$.
\end{definition}

A theorem echoing \cref{thm:clearcutWeak} can then be proved.

\begin{theorem}
	\label{thm:clearcutStrong}
	Given a decision situation $(T, \allargs, {\ileadsto}, {\ibeatse}, {\nibeatse})$, given $\clargs \subseteq \allargs$, if $\clargs$ is CAC, then i) the decision situation is clear-cut; ii) there exists an $\clargs$-operationally valid model $\eta$; iii) any $\clargs$-operationally valid model $\eta$ satisfies $T_i = T_\eta$.
\end{theorem}

This theorem is a strengthened version of \cref{thm:clearcutWeak} since it produces the same results based i) on the conditions encapsulated in the definition of CAC arguments, and ii) on $\clargs$-operational validity. Those conditions are implied by the ones assumed by \cref{thm:clearcutWeak}. Indeed, when the conditions of \cref{thm:clearcutWeak} hold, taking $\clargs = \allargs$ satisfies the conditions of \cref{thm:clearcutStrong}.%
\footnote{\Cref{thm:clearcutStrong} has an interesting corollary which permits to view our proposal as providing useful means to take account of the fact that knowledge evolves. In some cases it might be important, for example for efficiency reasons in contexts of limited resources, to investigate if a decision-aid provided before some discovery of new knowledge is still valid after the discovery. Take a decison-aid which has been provided using a set of argument $\clargs$ which is CAC with respect to the set of known arguments before the discovery $\allargs_\text{before}$ and using a $\clargs$-operationally valid model $\eta$. \Cref{thm:clearcutStrong} shows that, if we can prove that $\clargs$ is CAC with respect to the set of all the arguments $\allargs_\text{after}$ supplemented thanks to the new discovery, then there is no need to check the validity of $\eta$ again. We thank an anonymous reviewer for this observation.}

\section{Significance of the deliberated judgment framework for decision theory and the practice of decision analysis}
\label{sec:discussion}
\Cref{sec:core} displayed the conceptual core of our framework and \cref{sec:empirical} explained how this framework can be confronted to empirical reality. The present section reflects on the meaning, promises and limits of our approach. We start by pondering on how the various conditions spelled out in \cref{sec:empirical} can be interpreted (\cref{sec:meaning}). We then take a broader view to discuss how our framework relates to the larger literature in decision science (\cref{sec:pers}).

\subsection{The meaning of our conditions}
\label{sec:meaning}
In order to understand the precise meaning of the conditions of \cref{thm:clearcutWeak} and, more importantly, of \cref{thm:clearcutStrong}, an almost trivial but nonetheless very important first step is to spell out what it means if these conditions are \emph{not} fulfilled.

We already stressed that the conditions of \cref{thm:clearcutWeak} are certainly too strong to be fulfilled. The conditions of \cref{thm:clearcutStrong} are, by construction, much weaker. But still, there certainly are situations where they are not fulfilled. In such cases, we do not claim that decision analysis is impossible. Neither is our general framework, as presented in \cref{sec:core}, rendered bogus. The sole implication is that our approach to operational empirical validation cannot be implemented. This does not prevent, for example, the analyst from trying to identify directly decisive arguments, and this does not render irrelevant a decision analysis based on decisive arguments. Neither does this prevent completely other approaches to decision analysis to be implemented. The only implication is that a full-fledged implementation of our approach, including operational empirical validation, is not guaranteed to be possible in such situations. It is no part of our claim that our approach can be applied all the time and provides an all-encompassing framework liable to overcome all other approaches to decision analysis. Our approach has a specific domain of application.

Beyond these simple, negative comments, how are our conditions to be understood? In general terms, these various conditions can be interpreted in three different ways:
\begin{enumerate}[label=({\roman*})]
	\item \label{inter:axioms} as axioms capturing minimal properties concerning arguments and the way $i$ reasons,
	\item \label{inter:empir} as empirical hypotheses,
	\item \label{inter:rules} as rules governing the decision process (rules that $i$ can commit to abide by, or can consider to be well-founded safeguards for the proper unfolding of the process).
\end{enumerate}

\begin{example}[Budget reform (cont.)]
\label{ex:budgetInterpr}
We can now improve \cref{ex:budget} by relaxing the assumptions it contains. One can envisage in turn the three possibilities spelled out above.

In interpretation \ref{inter:axioms}, instead of assuming that $i$ always reasons in such a way that $\allargs$ in its entirety satisfies the conditions of \cref{thm:clearcutWeak}, we only assume that the set of argument $\args = \{\ar$, $\ar_{c1}$, $\ar_{c2}$, $\ar_{c1c}$, $\ar_{c2c}$, $\ar_{1, \text{reinstated}}$, $\ar_{2, \text{reinstated}}$, $\ar_\text{reinstated}\}$ is CAC.

In interpretation \ref{inter:empir}, we have to take advantage of empirical data to claim that the above set is CAC. Imagine, for example, that we have been able to show that the overwhelming majority of people does reason with respect to the arguments in this set in such a way that it can be considered CAC. This would provide strong empirical support to admit that this set can be considered CAC for the purpose of the decision process at issue (assuming the pragmatic interpretation of $\allargs$). In the present article, we leave aside the important difficulties that such empirical concrete applications would face.

In interpretation \ref{inter:rules}, the analyst would start by explaining to $i$ the content of the requirements encapsulated in the definition of a CAC set of arguments and ask her if she is willing to commit herself to reason in such a way as to fullfill these requirements when thinking about the arguments to be discussed in the process. For example, for the Answerability of the set of arguments (\cref{def:justUnstSet}), the analyst would ask $i$ if she would accept to commit not to change her mind depending on her mood or any other non-argumentative factor. Notice that $i$ might figure at some point that it was not a good idea after all to commit to these various things, and in such a case the decision analysis process would fail.
\end{example}

Some of the conditions of our theorems are arguably more congenial to a given interpretation. For example, it seems natural enough to interpret \cref{def:closed} as a rationality requirement of the kind that it makes sense to use as an axiom (interpretation \ref{inter:axioms}). By contrast, \cref{def:justifiableStrong} is the kind of condition that can easily be translated in the form of rules than decision-makers can be asked to abide by when they engage in a decision process (interpretation \ref{inter:rules}). 
By construction, \cref{def:justUnstSet,def:closedSet}
are weakened versions of the above stronger conditions. They accordingly inherit the preferred interpretation suggested above.
\Cref{def:setB.b,def:setB.lg} can easily be seen as empirical hypotheses (interpretation \ref{inter:empir}).

However, although it is tempting to draw such connections between specific conditions and specific interpretations, at a more abstract level all the conditions above can be interpreted in all three interpretations. The different conditions can even be interpreted differently in the context of different implementations. In the present, largely theoretical work, we want to leave all these possibilities open. Future, more applied works, should assess if and when these different interpretations can be used, in particular by elaborating and implementing the convenient empirical validation protocols in interpretation \ref{inter:empir} and the convenient participatory procedures in interpretation \ref{inter:rules}.

\subsection{The deliberated judgment framework in perspective}
\label{sec:pers}
Now that the meaning of the conditions of our theorems is clarified, we are in a firmer position to discuss the nature of our contribution to the literature.

The central, distinctive concept of our approach is the one of deliberated judgments of an individual. Deliberated judgments are the propositions that the individual herself considers based on decisive arguments, on due consideration. This formulation highlights the two key features of the concept.

	The first key feature is that deliberated judgments are the result of a careful examination of arguments and counter-arguments. This echoes the approach to the notion of rationality developed most prominently by \citet{habermas_theorie_1981}. In this approach, actions, attitudes or utterances can be termed “rational” so long as the actor(s) performing or having them can account for them, explain them and use arguments and counter-arguments to withstand criticisms that other people could raise against them. Variants of this vision of rationality play a key role in other prominent philosophical frameworks, such as \citeauthor{scanlon_what_2000}’s \citeyearpar{scanlon_what_2000} and \citeauthor{sen_idea_2009}’s \citeyearpar{sen_idea_2009}. Having in mind this approach to rationality, in the remainder of this discussion, we will therefore simply talk about “rationality” when referring to this first idea underlying our framework.

	The second key feature is that deliberated judgments are nevertheless the individual's own judgments, in the sense that they do not reflect the application of any exogenous criterion. This second idea can also be nicknamed, for brevity's stake, by simply talking about “non-paternalism”.

Our approach, when applied in a decision analysis perspective, requires admitting the soundness of these two normative notions of rationality and non-paternalism.

Our approach however also has a strong descriptive dimension, which is a direct implication of the very meaning of non-paternalism. Though we are interested in deliberated judgments rather than in the ``shallow'' preferences that individual spontaneously express, still the deliberated judgments that we are interested in are the ones of real, empirical individuals that are not constrained by our framework to adhere to a specific set of exogeneous stances. These descriptive aspects feed a normative approach that accordingly owes its normative credentials both to its normative foundations and to its reference to empirical reality.

Due to this double anchorage in normative and descriptive aspects, our approach opens avenues to overcome perennial difficulties facing decision theory concerning its descriptive vs. normative status. Indeed, our framework sets the stage for decision-aiding practices that could have a crucial strength as compared with more standard approaches, by including rigorous tests of whether individuals endorse or not various arguments and argumentative lines, thereby avoiding both actively advocating them (a purely normative approach) and leaving the individual in the ignorance of their existence (a purely descriptive approach). Decision analyses based on deliberated judgments thereby provide compelling reasons for the aided individual to think that the decisions he makes once he has been aided are better than the one he would have made otherwise. Such reasons are liable to play a key role in strengthening the legitimacy and validity of decision analysis -- two requirements largely discussed in the literature \citep{landry_model_1983, landry_model_1996}.

In order to illustrate this idea, it is useful to compare our framework to more classical approches, such as utility theory. Proponents of utility theory could claim that utility functions provide arguments that individuals will consider convincing \citep{savage_foundations_1972, morgenstern_reflections_1979, raiffa_back_1985}, and that therefore our approach will converge towards utility theory. However, the convincing power of utility-based arguments is debatable \citep{ellsberg_risk_1961, allais_so-called_1979}. Psychologists have tried to test it experimentally \citep{slovic_who_1974, maccrimmon_utility_1979}. But such tests can hardly be considered conclusive: the meaning of their results depends on how arguments have been presented to the individuals and on whether counter-arguments have been presented, as \citet{slovic_who_1974} themselves point out. Such a systematic confrontation with counter-arguments is precisely what our proposed framework allows to implement. 

The formal framework presented in this article will however only live up to its promises if empirical applications are developed. Researchers in artificial intelligence \citep{labreuche_general_2011} and persuasion \citep{carenini_generating_2006} have produced ways of “translating” formal Multi-Attribute Value Theory models into textual arguments, that could possibly provide promising tools to develop such applications.

\section*{Acknowledgements}
{
\setlength{\emergencystretch}{.5em}
We thank Denis Bouyssou, Cyril Hédoin, Jean-Sébastien Gharbi, André Lapied, Bernard Roy, Stéphane Deparis and two anonymous reviewers for very helpful comments.

}

\bibliography{philo-eco}

\appendix
\section{Proofs, and additional explanatory results}
\label{sec:proofs}
Our main goal in this section is to prove \cref{thm:clearcutStrong}. We do this by first proving that if a set $\clargs$ is CAC, then it includes enough decisive arguments to settle the issue (we will call such a set $\args \subseteq \allargs$ \emph{efficient}). This requires a few intermediate lemmas. Efficiency will bring a number of consequences of interest to us, among which \cref{thm:clearcutStrong}. As a second goal, we want to give some further results that help understand the relationship between the notions of clear-cut, validity and operational validity, existence of a CAC set of arguments, and efficiency.

Let us start with the formal definition of efficiency.
\begin{definition}[Efficiency]
	Given a decision situation $(T, \allargs, {\ileadsto}, {\ibeatse}, {\nibeatse})$ and $\args \subseteq \allargs$, $\args$ is \emph{efficient} iff
	$T_i = \ileadsto(S ∩ \overline{\ibeatse(\allargs)})$, and $t \notin T_i ⇔ \ileadstoinv(t) \subseteq \ibeatsst(S ∩ \overline{\ibeatse(\allargs)})$.
\end{definition}

Recall that $\overline{\ibeatse(\allargs)}$ designates the arguments not trumped by any argument, thus, the decisive arguments, and hence, $\ibeatsst(\args ∩ \overline{\ibeatse(\allargs)})$ designates the arguments always trumped by some decisive argument in $\args$.

In all this section, we assume we are given a decision situation $(T, \allargs, \ileadsto,\allowbreak \ibeatse,\allowbreak \nibeatse)$ and a subset of arguments $\clargs \subseteq \allargs$ (except in \cref{thm:clearSubsetEquivEfficient}).

Our strategy for proving that CAC implies efficiency, roughly speaking, involves excluding “undecided” situations from $\clargs$. For example, we want to show that it is impossible that an argument has no decisive argument trumping it in $\clargs$, but also fails to be defended in $\clargs$. We will do this by progressively promoting or degrading arguments, e.g., show that, in $\clargs$, if an argument is resistant (has no argument that decisively trumps it), then it must also be defended, and if it is defended, it must be replaceable by decisive arguments.

Define $\argscldec = \clargs ∩ \overline{\ibeatse(\allargs)}$ as the decisive arguments from $\clargs$.

Define an argument $\ar$ as finitely defended iff some finite set of arguments from $\argscldec$ defends it, thus, iff $\exists \args \subseteq \argscldec$ such that ${\ibeatseinv}(\ar) \subseteq {\ibeatse}(\args)$, $\args$ finite. Define $\argscldef$ as the arguments from $\clargs$ that are finitely defended.

Define $\argsrreplcldec \subseteq \allargs$ as the arguments that are replaceable by $\argscldec$. Recall that $\args$ replaces $\ar$ iff ${\ibeatse}(\ar) \subseteq {\ibeatse}(\args)$ and ${\ileadsto}(\ar) \subseteq {\ileadsto}(\args)$.

Define $\argsclres = \clargs ∩ \overline{\ibeatse(\argscldec)}$ as the resistant arguments from $\clargs$, namely, those not trumped by any argument from $\argscldec$.

Define $\argsreplclres \subseteq \allargs$ as the arguments that are essentially replaceable by $\argsclres$. Recall that $\args$ essentially replaces $\ar$ iff $(\ibeatse(\ar) ∩ \clargs) \subseteq {\ibeatse}(\args)$ and ${\ileadsto}(\ar) \subseteq {\ileadsto}(\args)$.

Similarly, $\argsreplcldec$ are the arguments essentially replaceable by $\argscldec$.

\begin{lemma}[$\argscldef \subseteq \argsrreplcldec$]
	If $\clargs$ is Closed under reinstatement and Answerable, then the arguments from $\clargs$ that are finitely defended are replaceable by decisive arguments from $\clargs$; formally: $\argscldef \subseteq \argsrreplcldec$. 
\end{lemma}
\begin{proof}
	The strategy for this proof is the following. If $\ar \in \argscldef$, some finite set of arguments defends $\ar$. 
	We wish to pick defenders one by one, replacing $\ar$ by applying Closed under reinstatement to $\ar$ and the chosen defender, obtaining an argument that fewer arguments trump, and then show that iterating the process yields a decisive argument replacing $\ar$.
	
	We need the following intermediate result. Assume a set of arguments $\args \subseteq \argscldec$ is given, together with an argument $\ar_1 \in \args$ and an argument $\ar^\text{r}_1 \in \clargs$ defended by $\args$. Then, there exists an argument $\ar^\text{r}_2 \in \clargs$ replacing $\ar^\text{r}_1$ and defended by $\args \setminus \{\ar_1\}$.
	
	Indeed, from Answerability, because $\ar_1 \in \argscldec$, $\ibeatse(\ar_1) = \ibeatsst(\ar_1)$. Also, as $\ar_1 \in \argscldec$, we can assume that $\ar_1 ≠ \ar^\text{r}_1$, otherwise $\ar^\text{r}_1 \in \argscldec$ and the result is obtained by taking $\ar^\text{r}_2 = \ar^\text{r}_1$. And $\ar_1$ does not trump $\ar^\text{r}_1$, otherwise $\ar^\text{r}_1$ is trumped by a decisive argument and thus not defended. We can thus apply Closed under reinstatement to $(\ar_1, \ar^\text{r}_1)$. 
	We obtain that for some $\ar^\text{r}_2 \in \clargs$, $\ar^\text{r}_2$ replaces $\ar^\text{r}_1$ and $\ibeatseinv(\ar^\text{r}_2) \subseteq \ibeatseinv(\ar^\text{r}_1) \setminus \ibeatse(\ar_1)$. Thus, $\args \setminus \{\ar_1\}$ defends $\ar^\text{r}_2$: any argument trumping $\ar^\text{r}_2$ already trumped $\ar^\text{r}_1$, hence, is trumped by $\args$ (because that set defends $\ar^\text{r}_1$), and is not trumped by $\ar_1$. This proves our intermediate result.
	
	Coming back to the main point, we know that a finite coalition $\args \subseteq \argscldec$ defends $\ar \in \clargs$. Define $\ar^\text{r}_1 = \ar$ and apply the intermediate result repetitively to obtain an argument $\ar^\text{r}_2 \in \clargs$ replacing $\ar$ and defended by $\args$ minus one element, then $\ar^\text{r}_3 \in \clargs$ replacing $\ar^\text{r}_2$, thus, replacing $\ar$ (because replacement is transitive) and defended by $\args$ minus two elements, and so on, until obtaining a replacer defended by $\emptyset$, thus, decisive.
\end{proof}

\begin{lemma}[$\allargs = \argsreplclres ∪ \ibeatse(\argsclres)$]
	If $\clargs$ is covering, any argument is either essentially replaceable by $\argsclres$, or attacked by an argument from $\argsclres$; formally: $\allargs = \argsreplclres ∪ \ibeatse(\argsclres)$.
\end{lemma}
\begin{proof}
	We consider in turn three sets whose union yields $\allargs$: $\overline{\clargs}$, $\clargs ∩ \ibeatse(\argscldec)$ and $\clargs ∩ \overline{\ibeatse(\argscldec)}$.
	
	First, $\overline{\clargs} \subseteq \argsreplclres ∪ \ibeatse(\argsclres)$: from covering, if $\ar \notin \clargs$, $\ar$ is unnecessary, and by definition, $\ar$ is unnecessary iff $\ar \in \argsreplclres$ or $\ar \in \ibeatse(\argsclres)$. 
	
	Second, $\clargs ∩ \ibeatse(\argscldec) \subseteq \ibeatse(\argscldec) \subseteq \ibeatse(\argsclres)$, because $\argscldec \subseteq \argsclres$.
	
	Third, $\clargs ∩ \overline{\ibeatse(\argscldec)} \subseteq \argsreplclres$, because $\clargs ∩ \overline{\ibeatse(\argscldec)} = \argsclres$ by definition.
	
	We have considered all three possible cases, and the conclusion obtains in all cases.
\end{proof}

\begin{lemma}[$\argsclres \subseteq \argscldef$]
	If $\clargs$ is CAC, any argument in $\clargs$ that has no argument that decisively trumps it is finitely defended; formally: $\argsclres \subseteq \argscldef$.
\end{lemma}
\begin{proof}
	Recall that the relation $Q$ is defined in Bounded length (\cref{def:setB.lg}) as $Q = (\ibeatse ∪ (\ibeatse \circ \ibeatse)) ∩ (\clargs × \clargs)$. Observe that, given any set $\args ≠ \emptyset$, Bounded Length forbids that $\forall \ar \in \args: \args ∩ Q^{-1}(\ar) ≠ \emptyset$. Otherwise, applying $Q^{-1}$ to an element of $\args$ would always yield some element in $\args$, and $Q^{-1}$ could then be applied any desired number of times starting from any $\ar \in \args$, thereby building a chain as long as desired. Accordingly, for any set $S$, Bounded Length imposes that if $\forall \ar \in \args: \args ∩ Q^{-1}(\ar) ≠ \emptyset$, then $\args = \emptyset$.
	
	Define $\args = \argsclres ∩ \overline{\argscldef}$. We show that, given any $\ar \in \args$, $\args ∩ Q^{-1}(\ar) ≠ \emptyset$. This suffices to obtain $\args = \emptyset$ and, therefore, our desired conclusion.
	
	Pick any $\ar \in \args$. Towards exhibiting an argument in $\args ∩ Q^{-1}(\ar)$, we want first to exhibit some argument $\ar'$ that is a) trumped by some argument $\ar^* \in \argsclres$, thus $\ar' \in \ibeatse(\argsclres)$; b) not trumped by any argument in $\argscldec$, thus $\ar' \notin \ibeatse(\argscldec)$; c) equal to $\ar$ or trumping $\ar$. As a second step, from the existence of such an $\ar'$ we will then prove that $\ar^*$, the particular trumping argument in part a), belongs to $\args$ (thanks to parts a) and b)), and belongs to $Q^{-1}(\ar)$ (thanks to part c)).
	
	Our first step thus amounts to show that some $\ar'$ satisfies our three conditions above.

	From $\ar \notin \argscldef$ and $\ar \in \clargs$, we know that $\ar$ is not finitely defended, and using the contrapositive of Bounded width, we obtain that $\ar$ is not infinitely defended either. Hence, by definition of defense, there exists some $\ar_1 \in \overline{\ibeatse(\argscldec)} ∩ \ibeatseinv(\ar)$.
	And, applying [$\allargs = \argsreplclres ∪ \ibeatse(\argsclres)$], either $\ar_1 \in \argsreplclres$, or $\ar_1 \in \ibeatse(\argsclres)$.
	
	If $\ar_1 \in \argsreplclres$, $\ar \in \ibeatse(\argsclres)$. Besides, because $\ar \in \args$, $\ar \in \argsclres$. Thus taking $\ar' = \ar$ satisfies our three conditions.
	
	And if $\ar_1 \in \ibeatse(\argsclres)$, because $\ar_1 \in \overline{\ibeatse(\argscldec)}$), taking $\ar' = \ar_1$ satisfies our three conditions.
	
	For our second step, consider an argument $\ar^* \in \argsclres$ that trumps $\ar'$ (we know this is possible thanks to part a)). Thanks to part b), we know that $\ar'$ is not trumped by any argument in $\argscldec$, and from [$\argscldef \subseteq \argsrreplcldec$], we know that if $\ar'$ was trumped by an argument in $\argscldef$, it would be trumped by an argument in $\argscldec$, thus, $\ar'$ is not trumped by any argument in $\argscldef$. Because $\ar^* \ibeatse \ar'$, we know that $\ar^* \notin \argscldef$. Thus, $\ar^* \in \args$. Finally, $\ar^* \ibeatse \ar$ or $\ar^* \ibeatse \ar' \ibeatse \ar$ (thanks to part c)), thus, $\ar^* \in Q^{-1}(\ar)$.
\end{proof}

\begin{lemma}[$\allargs = \argsreplcldec ∪ \ibeatse(\argscldec)$]
	If $\clargs$ is CAC, any argument is either essentially replaceable by decisive arguments from $\clargs$, or attacked by a decisive argument from $\clargs$; formally: $\allargs = \argsreplcldec ∪ \ibeatse(\argscldec)$.
\end{lemma}
\begin{proof}
	This follows from [$\allargs = \argsreplclres ∪ \ibeatse(\argsclres)$], [$\argsclres \subseteq \argscldef$] and [$\argscldef \subseteq \argsrreplcldec$].
\end{proof}

\begin{theorem}[CAC implies efficiency]
	If $\clargs$ is CAC, $\clargs$ is efficient.
	\label{thm:CACThusEfficient}
\end{theorem}
\begin{proof}
	We prove that $\ileadsto(\argsreplcldec) \subseteq \ileadsto(\argscldec) \subseteq T_i \subseteq \ileadsto(\overline{\ibeatsst(\argscldec)}) \subseteq \ileadsto(\argsreplcldec)$.%
	
	This proves the point, as it shows that
	\begin{enumerate}[label={\roman*}.]
		\item $T_i = \ileadsto(\argscldec)$, and
		\item $t \notin T_i ⇔ \ileadstoinv(t) \subseteq \ibeatsst(\argscldec)$, because $T_i = \ileadsto(\overline{\ibeatsst(\argscldec)})$.
	\end{enumerate}

	That $\ileadsto(\argsreplcldec) \subseteq \ileadsto(\argscldec) \subseteq T_i$ follows from the definitions of $\argsreplcldec$ and $T_i$.
	
	The next subset relation holds because if some decisive argument supports $t$, that argument is not in $\ibeatsst(\argscldec)$.

	Finally, Answerability mandates that $\ibeatse(\argscldec) \subseteq \ibeatsst(\argscldec)$, from which it follows that $\overline{\ibeatsst(\argscldec)} \subseteq \overline{\ibeatse(\argscldec)}$, and using [$\allargs = \argsreplcldec ∪ \ibeatse(\argscldec)$], $\overline{\ibeatse(\argscldec)} \subseteq \argsreplcldec$.
\end{proof}

\begin{theorem}[Validity of $\eta$]
	\label{thm:efficientThusValid}
	Assume $\clargs$ is efficient and $\eta$, a model of the decision situation, is $\clargs$-operationally valid. Then $T_i = T_\eta$.
\end{theorem}

\begin{proof}
	Recall that a model is \emph{$\clargs$-operationally valid} iff
	for all $(\ar \mleadsto t)$, $\ar \in \allargs$, we have $[\ar \ileadsto t]$ and $[\forall \ar_c \in \clargs: (\ar_c \nibeatse \ar) ∨ (\exists \ar_{cc} \in \allargs \suchthat \ar_{cc} \mbeats \ar_c ∧ \ar_{cc} \ibeatse \ar_c)]$, and when $t$ is not supported by $\eta$, $\forall \ar \in \clargs \suchthat \ar \ileadsto t: (\exists \ar_c \in \allargs \suchthat \ar_c \mbeats \ar ∧ \ar_c \ibeatse \ar)$.
	
	Consider $t \in T_\eta$. By definition, some $\ar \mleadsto t$. From operational validity of $\eta$, we obtain that $\ar \ileadsto t$ and $\forall \ar_c \ibeatsst \ar: \ar_c \notin \clargs ∩ \overline{\ibeatseinv(\allargs)}$ (because $[\ar_c \ibeatsst \ar ∧ \ar_c \in \clargs] ⇒ \ar_c \in \ibeatseinv(\allargs)$). Hence, $\ar \notin \ibeatsst(\clargs ∩ \overline{\ibeatseinv(\allargs)})$, thus $\ileadstoinv(t) \nsubseteq \ibeatsst(\clargs ∩ \overline{\ibeatseinv(\allargs)})$. Efficiency of $\clargs$ brings $t \in T_i$. 

	If $t \notin T_\eta$, from operational validity of $\eta$, no decisive argument in $\clargs$ may support $t$, equivalently, $t \notin \ileadsto(\clargs ∩ \overline{\ibeatse(\allargs)})$, and from efficiency, $t \notin T_i$.
\end{proof}

We can now prove \cref{thm:clearcutStrong}.
\begin{proof}[Proof of \cref{thm:clearcutStrong}]
	From [CAC implies efficiency], we obtain that $\clargs$ is efficient. It then follows from the efficiency of $\clargs$ that the decision situation is clear-cut and that a $\clargs$-operationally valid model exists. The last consequence is given by \cref{thm:efficientThusValid}.
\end{proof}

The following theorem may help clarify the relationship between efficiency, existence of CAC arguments, and the situation admitting a model as we conceive it.

\begin{theorem}[CAC subset equivalent to efficiency]
	\label{thm:clearSubsetEquivEfficient}
	Given a decision situation $(T, \allargs,\allowbreak {\ileadsto},\allowbreak {\ibeatse}, {\nibeatse})$ and a subset of arguments $\args \subseteq \allargs$, 
	there exists a set $\clargs \subseteq \args$ that is CAC iff
	$\args$ is efficient.
\end{theorem}
\begin{proof}
	From [CAC implies efficiency], if some set $\clargs \subseteq \args$ is CAC, then $\clargs$ is efficient, and because efficiency propagates to supersets, $\args$ is efficient.
	
	If $\args$ is efficient (thus, the decision situation is clear-cut), then a CAC subset $\clargs$ exists: suffices to choose as members of $\clargs$ only the decisive arguments required to support the justifiable propositions and trump the supporters $\ar \ileadsto t$ of untenable propositions. Observing that no arguments trump any argument in the resulting set (thus $\ar \ibeatse \clar$ for no $\ar \in \allargs, \clar \in \clargs$), most of the conditions for $\clargs$ to be CAC are immediately seen to be satisfied. About arguments $\ar \in \allargs \setminus \clargs$ being unnecessary, we only have to show that when $\ar \ileadsto t$, either $\ar$ is trumped by an argument from $\clargs$ that is not decisively trumped, or $\ar$ is essentially replaceable by arguments from $\clargs$. Indeed, by our construction of $\clargs$, if $\ar$ supports an accepted $t$, it is essentially replaceable, and otherwise, it is trumped by a decisive argument.
\end{proof}
\end{document}